\documentclass[11pt]{article}

\usepackage{fullpage}
\usepackage[round]{natbib}
\usepackage{algorithm}
\usepackage[noend]{algorithmic}
\usepackage{amsmath,amsthm,amsfonts,amssymb}
\usepackage{amsmath}
\usepackage{hyperref}
\usepackage{color}
\usepackage{mathrsfs}
\usepackage{enumitem}
\usepackage{bm}
\usepackage{multirow}
\usepackage{booktabs}
\usepackage{makecell}
\usepackage{graphicx}
\usepackage{subfigure}
\usepackage{caption}
\usepackage{thmtools}
\usepackage{thm-restate}
\usepackage{setspace}
\usepackage{makecell}
\definecolor{light-gray}{gray}{0.85}
\usepackage{colortbl}
\usepackage{xcolor}
\usepackage{hhline}

\newcommand{\defeq}{\mathrel{\mathop:}=}
\newcommand{\vect}[1]{\ensuremath{\mathbf{#1}}}

\newcommand{\argmax}{\mathop{\rm argmax}}

\newcommand{\rank}{\mathrm{rank}}

\newcommand{\poly}{\mathrm{poly}}

\newcommand{\E}{\mathbb{E}}

\renewcommand{\P}{\mathbb{P}}

\newcommand{\cO}{\mathcal{O}}

\newcommand{\N}{\mathbb{N}}
\newcommand{\R}{\mathbb{R}}

\newcommand{\M}{\mat{M}}

\renewcommand{\a}{\vect{a}}
\renewcommand{\o}{\vect{o}}

\newcommand{\cD}{\mathcal{D}}

\newcommand{\cS}{\mathcal{S}}
\newcommand{\cA}{\mathcal{A}}
\newcommand{\cB}{\mathcal{B}}

\newcommand{\up}[1]{\overline{#1}}

\newcommand{\algName}{\textsf{OMLE-Equilibrium}}
\newcommand{\algNameMulti}{\textsf{multi-step} \textsf{OMLE-Equilibrium}}
\newcommand{\algAdv}{\textsf{OMLE-Adversary}}
\newcommand{\optequi}{\textsf{Optimistic\_Equilibrium}}

\newcommand{\reg}{{\rm Regret}}

\renewcommand{\th}{^{\rm th}}

\newcommand{\Pidet}{\Pi^{\text{det}}}
\newcommand{\vv}[1]{\bar{#1}}
\usepackage{times}

\newtheorem{theorem}{Theorem}

\newtheorem{corollary}[theorem]{Corollary}
\newtheorem{remark}[theorem]{Remark}

\theoremstyle{definition}
\newtheorem{definition}[theorem]{Definition}

\newtheorem{assumption}{Assumption}

\renewcommand{\O}{\mathbb{O}}
\newcommand{\T}{\mathbb{T}}
\renewcommand{\o}{\vect{o}}

\renewcommand{\M}{\mathbb{M}}

\newcommand{\fA}{\mathfrak{A}}

\newfloat{subroutine}{htbp}{loa}
\floatname{subroutine}{Subroutine}

\begin{document}

\title{\textbf{Sample-Efficient Reinforcement Learning of Partially Observable Markov Games}}

\author{
  Qinghua Liu\thanks{Princeton University. Email: \texttt{qinghual@princeton.edu
}}
  \and
 Csaba Szepesvári\thanks{DeepMind and University of Alberta. Email: \texttt{szepesva@ualberta.ca}}
  \and
  Chi Jin\thanks{Princeton University. Email: \texttt{chij@princeton.edu}}
}
\date{}

\maketitle

\begin{abstract}
This paper considers the challenging tasks of Multi-Agent Reinforcement Learning (MARL) under partial observability, where each agent only sees her own individual observations and actions that reveal incomplete information about the underlying state of system. This paper studies these tasks under the general model of multiplayer general-sum Partially Observable Markov Games (POMGs), which is significantly larger than the standard model of Imperfect Information Extensive-Form Games (IIEFGs). We identify a rich subclass of POMGs---weakly revealing POMGs---in which sample-efficient learning is tractable. In the self-play setting, we prove that a simple algorithm combining optimism and Maximum Likelihood Estimation (MLE) is sufficient to find approximate Nash equilibria, correlated equilibria, as well as coarse correlated equilibria of weakly revealing POMGs, in a polynomial number of samples when the number of agents is small. In the setting of playing against adversarial opponents, we show that a variant of our optimistic MLE algorithm is capable of achieving sublinear regret when being compared against the optimal maximin policies. To our best knowledge, this work provides the first line of sample-efficient results for learning POMGs.

\end{abstract}


\section{Introduction}

This paper studies Multi-Agent Reinforcement Learning (MARL) under \emph{partial observability}, where each player tries to maximize her own utility via interacting with an unknown environment as well as other players. In addition, each agent only sees her own observations and actions, which reveal incomplete information about the underlying state of system. A large number of real-world applications can be cast into this framework: in Poker, cards in a player's hand are hidden from the other players; in many real-time strategy games, players have only access to their local observations; in multi-agent robotic systems, agents with first-person cameras have to cope with noisy sensors and occlusions. While practical MARL systems have achieved remarkable success in a set of partially observable problems including Poker \citep{brown2019superhuman}, Starcraft \citep{vinyals2019grandmaster}, Dota \citep{berner2019dota} and autonomous driving \citep{shalev2016safe}, the theoretical understanding of MARL under partial observability remains very limited.

The combination of partial observability with multiagency introduces a number of unique challenges. The non-Markovian nature of the observations forces the agent to maintain memory and reason about beliefs of the system state, all while exploring to collect information about the environment. Consequently, well-known complexity-theoretic results show that learning and planning in partially observable environments is statistically and computationally intractable even in the single-agent setting \citep{papadimitriou1987complexity,mundhenk2000complexity,vlassis2012computational,mossel2005learning}. 
The presence of interaction between multiple agents further complicates the partially observable problems. In addition to dealing with the adaptive nature of other players who can adjust their strategies according to the learner's past behaviors, the learner is further required to discover and exploit the information asymmetry due to the separate observations of each agent.

As a result, prior theoretical works on partially observable MARL have  been mostly focused on a small subset of problems with strong structural assumptions. For instance, the line of works on Imperfect Information Extensive-Form Games (IIEFG) \citep[see, e.g.,][]{zinkevich2007regret,gordon2007no,farina2021model,kozuno2021model} assumes tree-structured transition with small depth\footnote{The sample complexity of learning IIEFGs scale polynomially with respect to the number of information sets, which typically has an exponential growth in depth.} as well as a special type of emission which can be represented as \emph{information sets}. In contrast, this paper considers a more general mathematical model known as Partially Observable Markov Games (POMGs). POMGs are the natural extensions of both Partially Observable Markov Decision Processes (POMDPs)---the standard model for single-agent partially observable RL, and Markov Games \citep{shapley1953stochastic}---the standard model for fully observable MARL. Despite the complexity barriers of learning partially observable systems apply to POMGs, they are of a worst-case nature, which do not preclude efficient algorithms for learning interesting subclasses of POMGs. This motivates  us to ask the following question:

\vspace{-0.25ex}
\begin{center}
\textbf{Can we develop efficient algorithms that learn a rich class of POMGs?}
\end{center}
\vspace{-0.25ex}

In this paper, we provide the first positive answer to the highlighted question in terms of the \emph{sample efficiency}. \footnote{For computational efficiency, due to the inherent hardness of planning in POMDPs, all existing provable algorithms that learn large classes of POMDPs (single-agent version of POMGs) require super-polynomial time. We leave the challenge of computationally efficient learning for future work.} 
We identify a rich family of tractable POMGs---\emph{weakly revealing} POMGs (see Section \ref{sec:weakly}). The weakly revealing condition only requires the joint observations of all agents to reveal certain amount of information about the latent states, which is satisfied in many real-world applications. The condition rules out the pathological instances where no player has any information to distinguish latent states, which prevents efficient learning in the worst case. 

In the self-play setting where the algorithm can control all the players to learn the equilibria by playing against itself, this paper proposes a new simple algorithm---\emph{Optimistic Maximum Likelihood Estimation for Learning Equilibria} (\algName). As the name suggests, it combines optimism, MLE principles with equilibria finding subroutines. The algorithm provably finds approximate Nash equilibria, coarse correlated equilibria and correlate equilibria of any weakly-revealing POMGs using a number of samples polynomial in all relevant parameters. 

In the setting of playing against adversarial opponents, we measure  the performance of our algorithm by  comparing against the optimal maximin policies. We first prove that  learning in this setting is hard if each player can only see her own observations
and actions. Nevertheless, if the agent is allowed to access other players' observations and actions \emph{after} each episode of play (e.g., watch the replays of the games from other players' perspectives afterwards), then we can design a new algorithm \algAdv~which achieves sublinear regret.

To our best knowledge, this is the first line of provably sample-efficient results for learning rich classes of POMGs. Importantly, the classes of problems that can be learned in this paper are significantly larger than known tractable classes of MARL problems under partial observability.

\subsection{Technical novelty}\label{subsec:tech-novel}
This paper builds upon the recent progress in learning single-agent POMDPs \citep{Liu2022when}, which identifies the class of weakly revealing POMDPs and develops OMLE algorithm for learning the optimal policy. Besides obtaining a completely new set of results in the multi-agent setting, we here highlight a few contributions and technical novelties of this paper comparing to \cite{Liu2022when}.

\vspace{-0.5ex}

\begin{itemize}[leftmargin=*,itemsep=1pt]
\item This paper rigorously formulates the models, related concepts and learning objectives of multi-player general-sum POMGs, and provides the first line of sample-efficient learning results.
\item Extending the weakly revealing conditions into the multi-agent setting lead to two natural candidates: either (a) joint observations or (b) individual observations are required to weakly reveal the state information. This paper shows the former (the weaker assumption) suffices to guarantee tractability.
\item Results in the self-play setting requires careful design of optimistic planning algorithms that effectively address the game-theoretical aspects of the problem under partial observability. We achieve this by Subroutine \ref{subalg}, which is even distinct from the standard techniques for learning MGs.
\item The discussions and results in the setting of playing against adversarial opponents are completely new, and unique to the multi-agent setup. 
\end{itemize}


\subsection{Related Works}
\label{sec:related}

Reinforcement learning has been extensively studied in the single-agent fully-observable setting \citep[see, e.g.,][and the references therein]{azar2017minimax,dann2017unifying,jin2018q,jin2020provably, zanette2020learning,jiang2017contextual,jin2021bellman} . For the purpose of this paper, we focus on reviewing existing works on partially observable RL and  multi-agent RL in the \emph{exploration} setting.

\paragraph{Markov games} In recent years, there has been growing interest in studying Markov games \citep{shapley1953stochastic} --- the standard generalization of MDPs from the single-player setting to the multi-player setting. Various sample-efficient algorithms have been designed for either  two-player zero-sum MGs \citep[e.g.,][]{brafman2002r,wei2017online,bai2020provable,liu2021sharp,NEURIPS2020_172ef5a9,xie2020learning,jin2021power} or  multi-player general-sum MGs \citep[e.g.,][]{liu2021sharp,jin2021v,song2021can,daskalakis2022complexity}.
However, all these works rely on  the states being fully observable, while the POMGs studied in this paper allow states to be only partially observable, which strictly generalizes MGs.

\paragraph{POMDPs} POMDPs generalize  MDPs from the fully observable setting to the partially observable setting.
It is well-known that in POMDPs both planning \citep{papadimitriou1987complexity,vlassis2012computational} and model estimation \citep{mossel2005learning} are  computationally hard  in the worst case.
Besides, reinforcement learning of  POMDPs is also known to be statistically hard: \cite{krishnamurthy2016pac} proved that finding a near-optimal policy of a POMDP in the worst case requires a number of samples that is exponential in the episode length. The hard instances are those pathological POMDPs where the observations contain no useful information for identifying the system dynamics. 
Nonetheless, these hardness results are all in the worst-case sense and there are still many intriguing positive results on sample-efficient learning of subclasses of POMDPs.
For example, \cite{guo2016pac,azizzadenesheli2016reinforcement,jin2020sample} applied spectral methods to learning undercomplete POMDPs and \cite{Liu2022when} developed the optimistic MLE approach for learning both undercomplete and overcomplete POMDPs. We refer interested readers to \cite{Liu2022when} for a thorough review of existing results on POMDPs.

In terms of algorithmic design, our  algorithms build upon the optimistic MLE methodology developed in \cite{Liu2022when}. Compared to \cite{Liu2022when}, our main algorithmic contribution lies in the design of the optimistic equilibrium computation subroutine in the self-play setting and the optimistic maximin policy design in the adversarial setting. 
In terms of analysis, our proofs requrie new techniques tailored to controlling game-theoretic regret, in addition to the OMLE guarantees imported from \cite{Liu2022when}.
For more detailed explanations of our  technical contribution, 
please refer to Section \ref{subsec:tech-novel}.

\paragraph{Imperfect-information extensive-form games} 
In the literature on game theory, there is a long history of learning  Imperfect-Information  Extensive-Form Games with perfect recall (IIEFGs), \citep[see, e.g.,][]{zinkevich2007regret,gordon2007no,farina2020faster,farina2021model,kozuno2021model} and the references therein. IIEFGs can be viewed as special cases of POMGs with \emph{tree-structured} transition and \emph{deterministic} emission (which are also known as information sets). As a result, IIEFGs can not \emph{efficiently} represent POMGs with \emph{general} transition and \emph{stochastic} emission (see Appendix \ref{sec:iiefg}), and thus sample-efficient learning results for IIEFGs does not imply sample-efficient learning of POMGs. On the other hand, we show that all IIEFGs can be \emph{efficiently} represented by $1$-weakly revealing POMGs (see Appendix \ref{sec:iiefg}). Therefore, all algorithms and theoretical results developed in this paper can immediately used to learn IIEFGs with a polynomial sample complexity.  

\paragraph{Decentralized POMDPs} 
There is another classic model for studying multi-agent partially observable RL, named decentralized POMDPs \citep[e.g.,][]{oliehoek2012decentralized,oliehoek2016concise}, which is a special subclass of POMGs where all players share a common reward target. 
Compared to general POMGs, decentralized POMDPs can only simulate cooperative relations among players, while general POMGs can model both cooperative and competitive relations. 
Besides, most works \citep[e.g.,][]{nair2003taming,bernstein2005bounded,oliehoek2008optimal,szer2012maa,dibangoye2016optimally,liu2017learning,amato2019modeling} along this direction mainly focus on the computational complexity of planning with \emph{known} models or simulators instead of the sample efficiency of learning from interactions (as in this paper), which requires learning and estimating the unknown environment while balancing the tradeoff between exploration and exploitation.

\paragraph{Interactive POMDPs}
Another related model is Interactive POMDPs (I-POMDPs) \citep{gmytrasiewicz2005framework,doshi2009graphical}, which generalizes POMDPs to handle the presence of other agents by augmenting the latent state with behavior models of other agents. Since I-POMDPs are strictly more general than POMDPs, they are also less understood in theory than POMDPs. 
Noticably, I-POMDPs can only efficiently handle the problems where the number of possible models of other agents is not too large. 
For the POMGs model considered in this paper, even in the simplest setting of two-player zero-sum games, an agent can choose from  doubly exponentially many different history-dependent strategies. As a result, I-POMDPs cannot efficiently simulate POMGs without using a latent state space that is doubly exponentially large.


\section{Preliminary} \label{sec:prelim}

In this paper, we consider Partially Observable Markov Games (POMGs) in its most generic---multiplayer general-sum form. Formally, we denote a tabular episodic POMG with $n$ players by tuple $(H, \cS, \{\cA_i\}_{i=1}^n,$\\$ \{\cO_i\}_{i=1}^n; \T,\O,\mu_1; \{r_i\}_{i=1}^n)$, where $H$ denotes the length of each episode,  $\cS$  the state space with $|\cS| = S$, 
 $\cA_i$ denotes the action space for the $i\th$ player with $|\cA_i| = A_i$. We denote by $\bm{a}:=(a_{1},\cdots,a_{n})$ the joint actions of all $n$ players, and by $\cA := \cA_1 \times \ldots \times \cA_n$ the joint action space with $|\cA|=A=\prod_i A_i$. $\P = \{\P_h\}_{h\in[H]}$ is the collection of transition matrices, so that $\P_h ( \cdot | s, \bm{a})\in\Delta_\cS$ gives the distribution of the next state if joint actions $\bm{a}$ are taken at state $s$ at step $h$. 
 $\mu_1$ denotes the distribution of the initial state $s_1$. 
 $\cO_i$ denotes the observation space for the $i^{\rm th}$ player with $|\cO_i|=O_i$.  We denote by $\o:=(o_1,\ldots,o_n)$ the joint observations of all $n$ players, and by $\cO:=\cO_1\times\ldots\times\cO_n$ with $|\cO|=O=\prod_i O_i$.
 $\O = \{\O_h\}_{h\in[H]}\subseteq \R^{O\times S}$ is the collection of joint emission matrices, so that $\O_h ( \cdot | s)\in\Delta_\cO$ gives the emission distribution over the joint observation space $\cO$  at state $s$ and step $h$. 
 Finally $r_i = \{r_{i,h}\}_{h\in[H]}$ is the collection of known reward functions for the $i^{\text{th}}$ player, so that $r_{i,h}(o_i) \in[0, 1]$ gives the deterministic reward received by the $i^{\text{th}}$ player if she observes $o_i$  at step $h$. 
 \footnote{This is equivalent to assuming the reward information is contained in the observation.} We remark that since the relation among the rewards of different players can be arbitrary, this model of POMGs  subsumes both the cooperative and the competitive settings in partially observable MARL.

 In a POMG, the states are always hidden from all players, and each player only observes \textbf{her own individual observations and actions}. That is, each player can not see the observations and actions of the other players.
 At the beginning of each episode, the environment samlpes $s_1$ from $\mu_1$. At each step $h\in[H]$, each player $i$ observes her own observation $o_{i,h}$ where $\o_h:=(o_{1,h},\ldots,o_{n,h})$ are jointly sampled from $\O_h(\cdot\mid {s_h})$. Then each player $i$ receives reward $r_{i,h}(o_{i,h})$ and picks action $a_{i,h}\in\cA_i$ simultaneously.  After that the environment transitions to the next state
 $s_{h+1}\sim\P_h(\cdot | s_h, \bm{a}_h)$ where $\a_h:=(a_{1,h},\ldots,a_{n,h})$. The current episode terminates  immediately once $s_{H+1}$ is reached.

\paragraph{Policy, value function}
To define different types of polices, we extend the conventions in \emph{fully observable} Markov games \citep{jin2021v} to the partially observable settings. A (\emph{random}) \emph{policy} $\pi_i$ of the $i^{\rm th}$ player is a map $\pi_i: \Omega \times \bigcup_{h=1}^{H}\left((\cO_i \times \cA_i)^{h-1}\times \cO_i\right) \rightarrow \cA_i$, which maps a random seed $\omega$ from space $\Omega$ and a history of length $h\in[H]$---say $\tau_{i,h} := (o_{i,1}, a_{i,1}, \cdots, o_{i,h})$, to an action in $\cA_i$. To execute policy $\pi_i$, we first draw a random sample $\omega$ at the beginning of the episode. Then, at each step $h$, the $i^{\rm th}$ player simply takes action $\pi_{i}(\omega, \tau_{i,h})$. We note here $\omega$ is shared among all steps $h \in[H]$. $\omega$ encodes both the correlation among steps and the individual randomness of each step. We further say a policy $\pi_i$ is \emph{deterministic} if $\pi_{i}(\omega, \tau_{i,h}) = \pi_{i}(\tau_{i,h})$ which is independent of the choice of $\omega$.

 By definition, a random policy is equivalent to a mixture of deterministic policies because given a fixed $\omega$ the decision of $\pi_i$ on any history is deterministic. With slight abuse of notation, we use $\pi_i(\omega,\cdot)$ to refer to the deterministic policy realized by policy $\pi_i$ and a fixed $\omega$.
 We denote the set of all policies of the $i^{\rm th}$ player by $\Pi_i$ and the set of all deterministic ones by $\Pi_i^{\rm det}$.

A \emph{joint (potentially correlated) policy} is a set of policies $\{\pi_i\}_{i=1}^n$, where the same random seed $\omega$ is shared among all agents, which we denote as $\pi = \pi_1 \odot \pi_2 \odot \ldots \odot \pi_n$. We also denote $\pi_{-i} = \pi_1 \odot \ldots \pi_{i-1} \odot \pi_{i+1} \odot \ldots \odot \pi_n$ to be the joint policy excluding the $i\th$ player. A special case of joint policy is the \emph{product policy} where the random seed has special form $\omega = (\omega_1, \ldots, \omega_n)$, and for any $i\in [n]$, $\pi_i$ only uses the randomness in $\omega_i$, which is independent of remaining $\{\omega_j\}_{j\neq i}$, which we denote as $\pi = \pi_1 \times \pi_2 \times \ldots \times \pi_n$.

We define the value function $V^\pi_{i}$ as the expected cumulative reward that the $i^\text{th}$ player will receive if all players follow joint policy $\pi$:
\begin{equation} \label{eq:value_general_policy}
\textstyle V^{\pi}_{i}\defeq \E_{\pi}\left[\sum_{h =
        1}^H r_{i,h}(o_{i,h}) \right].
\end{equation}
where the expectation is taken over the randomness in the initial state, the transitions, the emissions, and the random seed  $\omega$ in policy $\pi$.

\paragraph{Best response and strategy modification}
For any strategy $\pi_{-i}$, the \emph{best response} of the $i^{\rm th}$ player is defined as a policy of the $i^\text{th}$ player, which is independent of the randomness in $\pi_{-i}$ and achieves the highest value for herself conditioned on all other players deploying $\pi_{-i}$. 
Formally, the best response is the maximizer of $\max_{\pi'_i} V_{i}^{\pi'_i \times \pi_{-i}}$ whose value we also denote as $V_{i}^{\dagger, \pi_{-i}}$ for simplicity. By its definition, we know the best response can always be achieved by \emph{deterministic} policies.

A \emph{strategy modification} for the $i^{\rm th}$ player is a map $\phi_i: \Pi_i^{\rm det} \rightarrow \Pi_i^{\rm det}$,
which maps a deterministic policy in $\Pi_i^{\rm det}$  to another one in it. For any such strategy modification $\phi_i$, we can naturally extend its domain and image to include random policies, i.e., define its extension $\phi_i: \Pi_i \rightarrow \Pi_i$  as follows: by definition, a random policy $\pi$ can be expressed as a mixture of deterministic policies, i.e., as $\pi(\omega, \cdot)$ (a deterministic policy for a fixed $\omega$) with a distribution over $\omega$. Then if we apply map $\phi_i$ on random policy $\pi$, we can define the resulting random policy (denoted as $\phi_i \diamond \pi_i$) as $\phi_i(\pi_i(\omega,\cdot))$ (again a deterministic policy for a fixed $\omega$) with the same distribution over $\omega$.
For any joint policy $\pi$, we define the best strategy modification of the $i^\text{th}$ player as the maximizer of $\max_{\phi_i} V_{i}^{(\phi_i \diamond \pi_i) \odot \pi_{-i}}$. 

Different from the best response, which is completely independent of the randomness in $\pi_{-i}$, the best strategy modification changes the policy of the $i\th$ player while still utilizing the shared randomness among $\pi_i$ and $\pi_{-i}$. Therefore, the best strategy modification is more powerful than the best response: formally one can show that 
$\max_{\phi_i} V_{i}^{(\phi_i \diamond \pi_i) \odot \pi_{-i}} \ge \max_{\pi'_i} V_{i}^{\pi'_i \times \pi_{-i}}$ for any policy $\pi$.

\subsection{Learning objectives}

We focus on three classic equilibrium concepts in game theory---Nash Equilibrium, Correlated Equilibrium (CE) and Coarse Correlated Equilibrium (CCE).
First, a Nash equilibrium is defined as a product policy in which no player can increase her value by changing only her own policy. Formally,
\begin{definition}[Nash Equilibrium]\label{def:NE}
A \emph{product} policy $\pi$ is a \textbf{ Nash equilibrium} if $V_{i}^{\dag, \pi_{-i}}=V_{i}^{\pi }$ for all $i\in[n]$. A \emph{product} policy $\pi$ is an $\epsilon$-approximate Nash equilibrium if $ V_{i}^{\dagger, \pi_{-i}}\le V_{i}^{\pi }+\epsilon$ for all $i\in[n]$.
\end{definition}
The Nash-regret of a sequence of product policies is the cumulative violation of the Nash condition. 
\begin{definition}[Nash-regret]
Let $\pi^k$ denote the (product) policy deployed by an algorithm in the $k^\text{th}$ episode. After a total of $K$ episodes, the Nash-regret is defined as
$$\textstyle
\reg_{\rm Nash}(K)= \sum_{k=1}^K \max_{i\in[n]}{( V_{i}^{\dag, \pi_{-i}^k}-V_{i}^{\pi^k } )}.
$$
\end{definition}

Second, a coarse correlated equilibrium is defined as a joint (potentially correlated) policy where no player can increase her value by unilaterally changing her own policy. Formally,
\begin{definition}[Coarse Correlated Equilibrium] \label{def:CCE}
    A \emph{joint} policy $\pi$ is a \textbf{CCE} if $V_{i}^{\dag, \pi_{-i}}\le V_{i}^{\pi }$ for all $i\in[n]$. A \emph{joint} policy $\pi$ is an $\epsilon$-approximate CCE if $ V_{i}^{\dagger, \pi_{-i}}\le V_{i}^{\pi }+\epsilon$ for all $i\in[n]$.
\end{definition}
The only difference between Definition \ref{def:NE} and Definition \ref{def:CCE} is that a Nash equilibrium  has to be a product policy while a CCE  can be correlated. Therefore,  CCE is a relaxed notion of Nash equilibrium, and a Nash equilibrium is always a CCE. Similarly, we can define the CCE-regret for a sequence of potentially correlated policies as  the cumulative vilolation of the CCE condition. 
\begin{definition}[CCE-regret]
    Let $\pi^k$ denote the  policy deployed by an algorithm in the $k^\text{th}$ episode. After a total of $K$ episodes, the CCE-regret is defined as
    $$\textstyle
    \reg_{\rm CCE}(K)= \sum_{k=1}^K \max_{i\in[n]}{( V_{i}^{\dag, \pi_{-i}^k}-V_{i}^{\pi^k } )}.
    $$
    \end{definition}

Finally, a correlated equilibrium is defined as a joint (potentially correlated) policy where no player can increase her value by unilaterally applying any strategy modification. Formally,
\begin{definition}[Correlated Equilibrium] \label{def:CE}
A joint policy $\pi$ is a \textbf{CE} if $\max_{\phi_i} V_{i}^{(\phi_i \diamond \pi_i) \odot \pi_{-i}}=V_{i}^{\pi }$ for all $i\in[n]$. A joint policy $\pi$ is an  $\epsilon$-approximate CE if $\max_{\phi_i} V_{i}^{(\phi_i \diamond \pi_i) \odot \pi_{-i}}\le V_{i}^{\pi } + \epsilon$ for all $i\in[m]$.
\end{definition}
In Partially Observable Markov games, we always have that a Nash equilibrium is a CE, and a CE is a CCE. 
Finally, we define the CE-regret to be the cumulative violation of the CE condition. 
\begin{definition}[CE-regret]
    Let $\pi^k$ denote the  policy deployed by an algorithm in the $k^\text{th}$ episode. After a total of $K$ episodes, the CE-regret is defined as
    $$\textstyle
    \reg_{\rm CE}(K)= \sum_{k=1}^K \max_{i\in[n]}
    \max_{\phi_i} {( V_{i}^{(\phi_i \diamond \pi_i^k) \odot \pi_{-i}^k}-V_{i}^{\pi^k } )}.
    $$
    \end{definition}


\section{Weakly Revealing Partially Observable Markov Games}
\label{sec:weakly}
In this section, we define the class of weakly revealing POMGs. 
To begin with, we consider undercomplete POMGs where there are more observations than hidden states, i.e., $O\ge S$.  Formally, the family of $\alpha$-weakly revealing POMGs  includes all POMGs, in which the $S^\text{th}$ singular value of each emission matrix  $\O_h$ is lower bounded by $\alpha>0$.
\begin{assumption}[$\alpha$-weakly revealing condition]\label{asp:under}
   There exists $\alpha>0$, such that $\min_h \sigma_{S}(\O_h)\ge\alpha$.
\end{assumption}
Assumption \ref{asp:under} is simply a robust version of the condition that  the rank of each emission matrix is  $S$, which  guarantees that no two different latent state mixtures can generate the same observation distribution, i.e., $\O_h\nu_1 \neq \O_h\nu_2$ for any different $\nu_1,\nu_2\in\Delta_\cS$. Intuitively, this guarantees that the \textbf{joint observations} of all agents contain sufficient information to distinguish any two different state mixtures. We remark that this is much weaker than requiring the individual observations of each agent contain sufficient information about the latent states. The weakly revealing condition is important in  excluding those pathological POMGs where the observations contain no useful information for identifying the key parts of model dynamics.

Note that Assumption \ref{asp:under} never holds in the overcomplete setting ($S>O$) as it is impossible to distinguish any two latent state mixtures by  only inspecting the observation distribution in  a single step. To address this issue, we can instead inspect the observations for $m$-consecutive steps. 
To proceed, we 
define the $m$-step emission-action matrices 
\[
\{\M_h\in\R^{(A^{m-1}O^m)\times S}\}_{h\in[H-m+1]}
\]
as follows: 
Given an observation sequence $\vv{\o}$ of length $m$, initial state $s$ and action sequence $\vv\a$ of length $m-1$,
we let $[\M_h]_{(\vv\a,\vv\o),s}$ be the probability of receiving  $\vv\o$ provided that the action sequence $\vv\a$ is used from state $s$ and step $h$: 
 \begin{equation}\label{defn:M}\textstyle
     [\M_h]_{(\vv\a,\vv\o),s}= \P(o_{h:h+m-1}=\vv\o \mid 
     s_h=s,a_{h:h+m-2} =\vv\a ),\quad  \forall(\vv\a,\vv\o,s)\in \cA^{m-1}\times\cO^m\times\cS.
 \end{equation}

Similar to the undercomplete setting, 
 the weakly-revealing condition in the over-complete setting simply assumes the $S^\text{th}$ singular value of each $m$-step emission-action matrix  is lower bounded.
\begin{assumption}[multistep $\alpha$-weakly revealing condition]\label{asp:over}
    There exists $m\in\N$, $\alpha>0$ such that   $\min_h\sigma_{S}(\M_h)\ge\alpha$ where $\M_h$ is the $m$-step emission matrix defined in \eqref{defn:M}.
\end{assumption}

Assumption \ref{asp:over} ensures that  $m$-step consecutive observations shall contain sufficient information to distinguish any two different latent state mixtures. Note that Assumption \ref{asp:under} is a special case of Assumption \ref{asp:over} with $m=1$. 
Finally, we remark that the single-agent versions of Assumption \ref{asp:under} and  \ref{asp:over} were first identified in \cite{jin2020sample} and \cite{Liu2022when} as sufficient conditions for sample-efficient learning of single-step and multi-step weakly revealing POMDPs (the single-agent version of POMGs), respectively.

\section{Learning Equilibria with Self-play}
In this section, we study the self-play setting where the algorithm can control all the players to learn the equilibria by playing against itself. We propose a new algorithm --- \emph{Optimistic Maximum Likelihood Estimation for Learning Equilibria} (\algName) that can provably find Nash equilibria, coarse correlated equilibria and correlate equilibria in any weakly-revealing partially observable Markov games using a number of samples polynomial in all relevant parameters.

\subsection{Undercomplete partially observable Markov games}

\begin{algorithm}[t]
    \caption{\algName}
 \begin{algorithmic}[1]\label{alg:under}
\STATE \textbf{Initialize:} $\cB^1 = \{ \hat \theta\in\Theta:~\min_h\sigma_{S}(\hat\O_h)\ge \alpha\}$, $\cD=\{\}$ 
\FOR{$k=1,\ldots,K$}
\STATE compute $\pi^k=$\textsf{Optimistic\_Equilibrium}$(\cB^k)$   \label{alg1-1}
    \STATE  follow $\pi^k$ to collect a trajectory $\tau^k=(\o_1^k,\a_1^k,\ldots,\o_H^k,\a_H^k)$ \label{alg1-2}
    \STATE add $(\pi^k,\tau^k)$ into $\cD$  and update \vspace{-3mm}
    $$
    \cB^{k+1} = \bigg\{\hat\theta \in \Theta: \sum_{(\pi,\tau)\in\cD} \log \P_{{\hat\theta}}^{\pi} (\tau)
    \ge \max_{ \theta' \in\Theta} \sum_{(\pi,\tau)\in\cD} \log \P^{\pi}_{{\theta'}}(\tau) -\beta \bigg\} \bigcap \cB^1$$\label{alg1-3}
    \vspace{-3mm}
    \ENDFOR
    \STATE output $\pi^{\rm out}$ that is sampled uniformly at random from $\{\pi^k\}_{k\in[K]}$
 \end{algorithmic}
 \end{algorithm}

\begin{subroutine}[t]
    \caption{\textsf{Optimistic\_Equilibrium}$(\cB)$}
    \label{subalg}
    \begin{algorithmic}[1]
        \FOR{$i\in[n]$}\label{alg0-1}
\STATE let $\up{V}_i\in\R^{|\Pidet_1|\times\cdots\times|\Pidet_n|}$  
with its $\pi^{\rm th}$  entry equal to $\sup_{\hat\theta\in\cB} V_{i}^\pi(\hat\theta)$ for $\pi\in\Pi_1\times\cdots\times\Pi_n$
\ENDFOR\label{alg0-2}
\STATE return $\textsc{Equilibrium}(\up{V}_1,\ldots,\up{V}_n)$
 \label{alg0-3}
    \end{algorithmic}
\end{subroutine}

We first present the algorithm and results for learning undercomplete POMGs under Assumption~\ref{asp:under}. 
We will see in the later section that 
with a minor modification the same  algorithm also applies to learning overcomplete POMGs under Assumption~\ref{asp:over}.

\paragraph{Algorithm description} 
To condense notations, we will use $\theta=(\T,\O,\mu_1)$ to denote the parameters of a POMG. Given a policy $\pi$ and a trajectory $\tau$, we denote by $V^\pi_{i}(\theta)$ the $i^{\rm th}$ player's value and by $\P^\pi_\theta(\tau)$  the probability of observing trajectory $\tau$, both under policy $\pi$ in the POMG model  parameterized by $\theta$. 
We describe \algName~in Algorithm \ref{alg:under}. 
In each episode, the algorithm executes the following two key steps: 
\begin{itemize}
    \item \textbf{Optimistic equilibrium computation} (Line \ref{alg1-1})  
    We first invoke \optequi\\ 
    (Subroutine \ref{subalg}) with confidence set  $\cB^k$ to compute a joint (potentially correlated) policy $\pi^k$. Formally, subroutine \optequi$(\cB^k)$ consists of two components:
    \begin{itemize}
        \item \textbf{Optimistic value estimation} (Line \ref{alg0-1}-\ref{alg0-2} of Subroutine \ref{subalg})
        For each player $i\in[n]$ and  deterministic joint policy $\pi\in\Pidet_1\times\cdots\times\Pidet_n$, we compute an upper bound $\up{V}_i^\pi$ for the $i^{\rm th}$ player's value  under policy  $\pi$ by using the most optimistic POMG  model in the  confidence set $\cB^k$.

    \item \textbf{Equilibria computation} (Line \ref{alg0-3} of Subroutine \ref{subalg}) Given the optimistic value estimates for all deterministic  joint policies and all players, we can view the POMG as a normal-form game where  the $i^{\rm th}$ player's pure strategies consist of all her deterministic policies (i.e.,  $\Pidet_i$) and the payoff she receives under a joint deterministic policy $\pi\in\Pidet_1\times\cdots\times\Pidet_n$ is equal to the corresponding optimistic value estimate $\up{V}_i^\pi$. 
    Then we compute a \textsc{Equilibrium} $\pi^k$  for this normal-form game,  which is a mixture of all the deterministic joint policies in $\Pidet_1\times\cdots\times\Pidet_n$.
\end{itemize}

    \item \textbf{Confidence set update} (Line \ref{alg1-2}-\ref{alg1-3}) We first follow $\pi^k$ to collect a trajectory, and then  utilize the newly collected data to update the model confidence set via MLE principle. 
\end{itemize}
Here we highlight two algorithmic designs in \algName: the flexibility of equilibrium computation and the MLE confidence set construction.
In the step of equilibrium computation, we can choose \textsc{Equilibrium} to be  Nash equilibrium or correlated equilibrium (CE) or coarse correlated equilibrium (CCE) of the normal-form game depending on the target type of  equilibrium we aim to learn for the POMG. 
With regard to the confidence set design, we adopt the idea from  \cite{Liu2022when} to include all the POMG models whose likelihood on the historical data is close to the maximum likelihood. This can be viewed as a relaxation of the classic MLE method, with  the degree of relaxation controlled by parameter $\beta$. One important benefit of this relaxation is that although the groundtruth POMG model is in general not a solution of MLE, its likelihood ratio is rather close to the maximal likelihood. By doing so, we can guarantee the true model is included in the confidence set with high probability. Finally, we remark that Algorithm \ref{alg:under} is computationally inefficient in general due to the steps of optimistic value estimation and equilibrium computation.

\paragraph{Theoretical guarantees} Below we present the main theorem for \algName.

\begin{theorem}\label{thm:under}
\emph{(Regret of \algName)}
    Under Assumption \ref{asp:under},
	there exists an absolute constant $c$ such that for any $\delta\in(0,1]$ and $K\in\N$, Algorithm \ref{alg:under} with $\beta = c\left(H(S^2A+SO)\log(SAOHK)+\log(K/\delta)\right)$ and \textsc{Equilibrium} being one of $\{\text{Nash, CCE, CE}\}$ satisfies (respectively) that  
	with probability at least $1-\delta$, 
	$$\textstyle \reg_{\{\rm Nash,CCE,CE\}}(k) \le
    \poly(S,A,O,H,\alpha^{-1},\log(K\delta^{-1}))\cdot\sqrt{k} \qquad \text{ for all }k\in[K].$$
\end{theorem}
Theorem \ref{thm:under} claims that if all players follow \algName, then the cumulative \{Nash,CCE,CE\}-regret  is upper bounded by $\tilde{\cO}(\sqrt{k})$ for any weakly-revealing POMGs that satisfy Assumption \ref{asp:under}, where the growth rate w.r.t $k$ is optimal.   By the standard online-to-batch conversion, it directly implies the following sample complexity result: 
\begin{corollary}\emph{(Sample Complexity of \algName)}\label{cor:sample-under}
    Under the same setting as Theorem \ref{thm:under}, when
     $K \ge  \poly(S,A,O,H,\alpha^{-1},\log(\epsilon^{-1}\delta^{-1}))\cdot\epsilon^{-2}$, 
     then with probability at least $1/2$, $\pi^{\rm out}$  is an $\epsilon$-$\{\text{Nash, CCE, CE}\}$ policy.
    \end{corollary}
Here  the dependence on the precision parameter $\epsilon$ is optimal. 
Finally, notice that the upper bound in Theorem \ref{thm:under} depends polynomially on the inverse of $\alpha$ --- a lower bound for the minimal singular value of the joint emission matrix $\O_h$ in Assumption \ref{asp:under}. This dependence is shown to be unavoidable even in the single-player setting (POMDPs) \citep{Liu2022when}.

\subsection{Overcomplete partially observable Markov games}

In this subsection, we extend \algName~to the more challenging setting of learning overcomplete POMGs, where there can be much less observations than latent states. We prove that a simple variant of \algName~still enjoys polynomial sample-efficiency guarantee for learning any multi-step weakly revealing POMGs.

\begin{algorithm}[t]
    \caption{\algNameMulti}
 \begin{algorithmic}[1]\label{alg:over}
\STATE \textbf{Initialize:} $\cB^1 = \{ \hat \theta\in\Theta:~\min_h\sigma_{S}(\hat\M_h)\ge \alpha\}$, $\cD=\{\}$ 
\FOR{$k=1,\ldots,K$}
\STATE compute $\pi^k=$\textsf{Optimistic\_Equilibrium}$(\cB^k)$  
 \label{alg2-3}
    \FOR{$h=0,\ldots,H-m$} \label{alg2-4}
    \STATE execute policy $\pi^k_{1:h}\circ \text{uniform}(\mathcal{\cA})$  to collect a trajectory $\tau^{k,h}$ \\
    then add  $(\pi^k_{1:h}\circ \text{uniform}(\mathcal{\cA}),\tau^{k,h})$ into $\cD$ 
   \ENDFOR\label{alg2-5}
    \STATE update \vspace{-4mm}
    $$
    \cB^{k+1} = \bigg\{\hat\theta \in \Theta: \sum_{(\pi,\tau)\in\cD} \log \P_{{\hat\theta}}^{\pi} (\tau)
    \ge \max_{ \theta' \in\Theta} \sum_{(\pi,\tau)\in\cD} \log \P^{\pi}_{{\theta'}}(\tau) -\beta \bigg\} \bigcap \cB^1
    $$ \vspace{-3mm}
    \ENDFOR
    \STATE output $\pi^{\rm out}$ that is sampled uniformly at random from $\{\pi^k\}_{k\in[K]}$
 \end{algorithmic}
 \end{algorithm}

\paragraph{Algorithm description} 
We describe the multi-step generalization of \algName~in Algorithm \ref{alg:over}, which inherits the key designs from Algorithm \ref{alg:under} and additionally makes two important modifications to address the challenge of insufficient information from single-step observation. 
The first change is to utilize a more active sampling strategy for exploration. Instead of simply following the optimistic policy $\pi^k$, we will iteratively execute  $H-m+1$  policies of form $\pi^k_{1:h}\circ \text{uniform}(\fA)$  where the players first follow policy $\pi^k$ from step $1$ to step $h$, then pick actions uniformly at random to  finish the remaining $H-h$ steps.
Intuitively, by actively trying random action sequences after executing policy $\pi^k$,  the algorithm can acquire more information about the system dynamics corresponding to those latent states that are frequently visited by $\pi^k$, and therefore help address the challenge of lacking sufficient  information from single-step observation.
The second change made by Algorithm \ref{alg:over} is that in constructing the confidence set, we require the minimal singular value of the multistep emission matrix to be lower bounded, which enforces the multistep weakly revealing condition in Assumption \ref{asp:over}.

\paragraph{Theoretical guarantee} Below we present the main theorem for \algNameMulti.

\begin{theorem}
    \emph{(Total suboptimality of \algNameMulti)}
    \label{thm:over}
    Under Assumption \ref{asp:over},
	there exists an absolute constant $c$ such that for any $\delta\in(0,1]$ and $K\in\N$, Algorithm \ref{alg:over} with 
	$$\beta = c\left(H(S^2A+SO)\log(SAOHK)+\log(K/\delta)\right)$$ and \textsc{Equilibrium} being one of $\{\text{Nash, CCE, CE}\}$ satisfies (respectively) that  
	with probability at least $1-\delta$, 
	$$\textstyle \reg_{\{\rm Nash,CCE,CE\}}(k) \le  \poly(S,A^m,O,H,\alpha^{-1},\log(K\delta^{-1}))\cdot\sqrt{k} \qquad \text{ for all }k\in[K],$$
    where the regret is computed for policy $\pi^1,\ldots,\pi^k$.
\end{theorem}
Theorem \ref{thm:over} claims that the total  \{Nash,CCE,CE\}-``regret'' (that are computed on policy $\pi^1,\ldots,\pi^k$) of  \algNameMulti~is upper bounded by $\tilde{\cO}(\sqrt{k})$ for any multi-step weakly revealing POMGs satisfying Assumption \ref{asp:over}. We remark that, strictly speaking, Theorem \ref{thm:over} is not a standard regret guarantee since the policies executed by \algNameMulti~are compositions of  $\pi^1,\ldots,\pi^k$ and random actions, instead of purely $\pi^1,\ldots,\pi^k$. 
Nevertheless, we can still utilize the standard online-to-batch conversion to obtain the following sample complexity guarantee: 
\begin{corollary} \emph{(sample complexity of \algNameMulti)}
    Under the same setting as Theorem \ref{thm:over}, when $K \ge \poly(S,A^m,O,H,\alpha^{-1},\log(\epsilon^{-1}\delta^{-1}))\cdot\epsilon^{-2},$   then with probability at least $1/2$, $\pi^{\rm out}$  is an $\epsilon$-$\{\text{Nash, CCE, CE}\}$ policy.
    \end{corollary}
Here the dependence on the precision parameter $\epsilon$ is optimal up to poly-logarithmic factors. 
Finally, observe that the sample complexity has $A^m$ dependency that is exponential in $m$, which is unavoidable in general even in the single-player setting (POMDPs) \citep{Liu2022when}. Nonetheless, in order to make $\min_h\rank(\M_h)=S$ possible (Assumption \ref{asp:over}), we only need to make $(OA)^m \gtrsim S$, i.e., $m \gtrsim \log S$ which is very small. In this paper, when we claim the sample complexity is polynomial, we consider $m$ to be small enough so that $A^m \le \poly(S,A,O,H, \alpha^{-1})$.


\section{Playing against Adversarial Opponents}

In this section, we turn to the online setting where the learner only controls a single player  and the remaining players can execute arbitrary strategies. In this setting, we no longer target at learning  game-theoretic  equilibria because if other players keep playing some highly suboptimal policies then  the learner may never be able to explore the environment thoroughly and thus lacks sufficient information to compute  equilibria. 
Instead, we consider the standard goal for  online setting, which is to achieve low regret in terms of cumulative rewards even if all other players play adversarially against the learner. 
Without loss of generality, we assume the learner only controls the $1^{\rm st}$ player throughout this section.

\subsection{Statistical hardness for the standard setting}\label{subsec:adv-standard}

We first consider the standard POMG setting where each player can only  observe her \emph{own} observations and actions. We prove that achieving low regret in this setting is impossible in general even if (i) the POMG is two-player zero-sum and satisfies Assumption \ref{asp:under} with $\alpha=1$, (ii) the opponent keeps playing a fixed deterministic policy \emph{known} to the learner, and (iii) the only  parts of  the model unknown to the learner are the emission matrices.

\begin{theorem}\label{thm:lowerbound}
    For any $L,k\in\N^+$, there exist (i) a  two-player zero-sum POMG  of size $S,A,O,H=\cO(L)$ and satisfying Assumption \ref{asp:under} with $\alpha=1$, and (ii) a fixed opponent who keeps playing a known deterministic policy $\pi_2$, so that with probability at least $1/2$
    $$\textstyle
    \sum_{t=1}^k \left(\max_{\tilde \pi_1}\min_{\tilde \pi_2} V_1^{\tilde\pi_1 \times \tilde\pi_2} - V_{1}^{\pi^t_1 \times \pi_2}\right) \ge \Omega\left(\min\{2^L,k\} \right),
    $$
    where $\pi_1^t$ is the policy played by the learner in the $t^{\rm th}$ episode. 
    \end{theorem}
Theorem \ref{thm:lowerbound} claims that when the learner is not allowed to access the opponent's observations and actions, there exists exponential regret lower bound for competing with the max-min value (i.e., Nash value in two-player zero-sum POMGs) even in the very benign scenario as described above. We remark that this lower bound directly implies competing with the best fixed policy in hindsight is also hard because the max-min value is always no larger than the value of the best-response to $\pi_2$:
$$ \textstyle \max_{\tilde \pi_1}\min_{\tilde \pi_2} V_1^{\tilde\pi_1 \times \tilde\pi_2}\le  \max_{\tilde\pi_1} V_{1}^{\tilde \pi_1\times \pi_{2}}=  V_{1}^{\dagger, \pi_{2}}.$$

\begin{algorithm}[t]
    \caption{\algAdv}
 \begin{algorithmic}[1]\label{alg:adv}
\STATE \textbf{Initialize:} $\cB^1 = \{ \hat \theta\in\Theta:~\sigma_{S}(\hat\O)\ge \alpha\}$, $\cD=\{\}$ 
\FOR{$k=1,\ldots,K$}
    \STATE learner computes 
    $(\cdot, \pi^k_1) = \argmax_{\hat\theta\in\cB^k, \hat\pi_1\in\Pi_1}\min_{\hat\pi_{-1}\in\Pi_{-1}} V^{\hat\pi_1 \times \hat\pi_{-1}}_{1}(\hat\theta)$ \label{alg:3-1}
    \STATE opponents pick policies $\pi^k_{-1}$  \label{alg:3-2}
    \STATE execute policy $\pi^k=\pi^k_1\times\pi^k_{-1}$ to collect $\tau^k=(\o_1^k,\a_1^k,\ldots,\o_H^k,\a_H^k)$  
    \STATE  add $(\pi^k,\tau^k)$ into $\cD$ and update \label{alg:3-3}    \vspace{-2mm}
    \begin{equation}\label{eq:adv-Bset}
    \cB^{k+1} = \bigg\{\hat\theta \in \Theta: \sum_{(\pi,\tau)\in\cD} \log \P_{{\hat\theta}}^{\pi} (\tau)
    \ge \max_{ \theta' \in\Theta} \sum_{(\pi,\tau)\in\cD} \log \P^{\pi}_{{\theta'}}(\tau) -\beta \bigg\} \bigcap \cB^1
    \end{equation}    
    \vspace{-3mm}
    \ENDFOR
 \end{algorithmic}
 \end{algorithm}

\subsection{Positive results for the game-replay setting}
\label{subsec:adv-pos}

In this section, we consider the game-replay setting where \emph{after} each episode of play, every player will reveal their observations and actions in this episode to other players. In other words, every player is able to observe the whole trajectory $\tau^k=(\o_1^k,\a_1^k,\ldots,\o_H^k,\a_H^k)$ \emph{after} the $k^{\rm th}$ episode is finished. 
The motivation for considering this setting is in many real-world games, e.g., Dota, StarCraft and Poker, players are usually allowed to watch the replays of the games they have played, in which they can freely view other players' observations and actions.  
Below, we show that a simple variant of \algName~enjoys sublinear regret when playing against adversarial opponents.

\paragraph{Algorithm description} We provide the formal description of \algAdv~in Algorithm \ref{alg:adv}. 
Same as \algName, \algAdv~utilizes the  relaxed MLE approach to construct the confidence set. 
The key modification lies in the computation of player $1$'s (stochastic)  policy $\pi_1^k$ (Line \ref{alg:3-1}). Specifically, the learner will compute the most optimistic model ${\theta}^k$ in the confidence set $\cB^k$ by examining  player $1$'s max-min value in each model. Then choose player $1$'s policy to be the one with the highest value under ${\theta}^k$, assuming all other players jointly play against player $1$.

Finally, we remark that although the confidence set construction in Line \ref{alg:3-3} seems to involve the joint policies $\pi$ of all players, the confidence set itself is in fact independent of the joint policies $\pi$. Therefore, the $1^{\rm st}$ player can still construct the confidence set without knowledge of other players' policies. This is because the dependency of the loglikelihood function on policy $\pi$ are equal on both sides of \eqref{eq:adv-Bset}, and thus they cancel with each other. Formally, for any  $\hat\theta,\theta'\in\Theta$, we have
$$
\textstyle 
\sum_{t=1}^{k} \left(\log \P_{{\hat\theta}}^{\pi^t} (\tau^t)
-  \log \P^{\pi^t}_{{\theta'}}(\tau^t) \right)  = \sum_{t=1}^{k} \left(\log \P_{{\hat\theta}} (\o_{1:H}^t\mid \a_{1:H}^t)
-  \log \P_{{\theta'}}(\o_{1:H}^t\mid \a_{1:H}^t) \right).
$$

\paragraph{Theoretical guarantees}
Below we present the main theorem for \algAdv.

\begin{theorem}\emph{(Regret of \algAdv)}\label{thm:adv}
    Under Assumption \ref{asp:under},
	there exists an absolute constant $c$ such that for any $\delta\in(0,1]$ and $K\in\N$, Algorithm \ref{alg:adv} with $\beta = c\left(H(S^2A+SO)\log(SAOHK)+\log(K/\delta)\right)$ satisfies  that  
	with probability at least $1-\delta$, 
	$$\textstyle
    \sum_{t=1}^k \left(\max_{\tilde \pi_1}\min_{\tilde \pi_2} V_1^{\tilde\pi_1 \times \tilde\pi_2}- V_{1}^{\pi^t}\right) 
    \le\poly(S,A,O,H,\alpha^{-1},\log(K\delta^{-1}))\cdot \sqrt{k} \quad \text{for all $k\in[K]$}.
    $$
\end{theorem}
Theorem \ref{thm:adv} claims that  the regret of \algAdv~is upper bounded by $\tilde{\cO}(\sqrt{k})$  in any weakly revealing POMGs that satisfy Assumption \ref{asp:under}, no matter what adversarial strategies other players might take.
Here the regret is defined by comparing the cumulative rewards received by player $1$ to the max-min value 
$\max_{\tilde \pi_1}\min_{\tilde \pi_2} V_1^{\tilde\pi_1 \times \tilde\pi_2}$ that is the largest value she could receive if all other players jointly play against her. 
Notice that this regret is  weaker than the typical version of regret considered in online learning literature, which typically competes with the best response in hindsight, i.e., 
$$
\textstyle 
\max_{\tilde\pi_1}\sum_{t=1}^k \left( V_{1}^{\tilde\pi_1\times\pi_{-1}^t}- V_{1}^{\pi^t}\right).
$$
Therefore, it is natural to ask whether we can obtain similar sublinear regret in terms of the above regret definition. Unfortunately, previous work \citep{liu2022learning}  proved that there exists exponential regret lower bound for competing with the best response in hindsight even in  \emph{fully observable} two-player zero-sum Markov games, which are special cases of POMGs  satisfying Assumption \ref{asp:under} with $\alpha=1$. As a result, achieving low regret in the above sense is also intractable in POMGs.

\paragraph{Negative results for generalization to multi-step weakly revealing POMGs}So far, we only derive the positive result (Theorem \ref{thm:adv}) for  single-step weakly revealing POMGs. 
A reader might wonder whether similar results can be obtained in the more general setting of multi-step weakly revealing POMGs. 
Unfortunately, this generalization turns out to be  impossible in general, even if (i) the POMG satisfies Assumption \ref{asp:over} with $m=2$ and $\alpha\ge1$, and (ii) the learner can directly observe the  opponents'  actions and observations. 
We defer the formal statement of this hardness result and its proof to Appendix \ref{sec:over-hard}.
\section*{Acknowledgements}
 Csaba Szepesv\'ari gratefully acknowledges the funding from Natural Sciences and Engineering Research Council (NSERC) of Canada, ``Design.R AI-assisted CPS Design'' (DARPA)  project and the Canada CIFAR AI Chairs Program for Amii.
Chi Jin gratefully acknowledges the Project X innovation fund from Princeton.
\bibliography{ref} 

\begin{thebibliography}{50}
\providecommand{\natexlab}[1]{#1}
\providecommand{\url}[1]{\texttt{#1}}
\expandafter\ifx\csname urlstyle\endcsname\relax
  \providecommand{\doi}[1]{doi: #1}\else
  \providecommand{\doi}{doi: \begingroup \urlstyle{rm}\Url}\fi

\bibitem[Brown and Sandholm(2019)]{brown2019superhuman}
Noam Brown and Tuomas Sandholm.
\newblock Superhuman {{AI}} for multiplayer poker.
\newblock \emph{Science}, 365\penalty0 (6456):\penalty0 885--890, 2019.

\bibitem[Vinyals et~al.(2019)Vinyals, Babuschkin, Czarnecki, Mathieu, Dudzik,
  Chung, Choi, Powell, Ewalds, Georgiev, et~al.]{vinyals2019grandmaster}
Oriol Vinyals, Igor Babuschkin, Wojciech~M Czarnecki, Michael Mathieu, Andrew
  Dudzik, Junyoung Chung, David~H Choi, Richard Powell, Timo Ewalds, Petko
  Georgiev, et~al.
\newblock Grandmaster level in {StarCraft II} using multi-agent reinforcement
  learning.
\newblock \emph{Nature}, 575\penalty0 (7782):\penalty0 350--354, 2019.

\bibitem[Berner et~al.(2019)Berner, Brockman, Chan, Cheung, Dkebiak, Dennison,
  Farhi, Fischer, Hashme, Hesse, et~al.]{berner2019dota}
Christopher Berner, Greg Brockman, Brooke Chan, Vicki Cheung, Przemyslaw
  Dkebiak, Christy Dennison, David Farhi, Quirin Fischer, Shariq Hashme, Chris
  Hesse, et~al.
\newblock Dota 2 with large scale deep reinforcement learning.
\newblock \emph{arXiv preprint arXiv:1912.06680}, 2019.

\bibitem[Shalev-Shwartz et~al.(2016)Shalev-Shwartz, Shammah, and
  Shashua]{shalev2016safe}
Shai Shalev-Shwartz, Shaked Shammah, and Amnon Shashua.
\newblock Safe, multi-agent, reinforcement learning for autonomous driving.
\newblock \emph{arXiv preprint arXiv:1610.03295}, 2016.

\bibitem[Papadimitriou and Tsitsiklis(1987)]{papadimitriou1987complexity}
Christos~H Papadimitriou and John~N Tsitsiklis.
\newblock The complexity of {{M}arkov} decision processes.
\newblock \emph{Mathematics of operations research}, 12\penalty0 (3):\penalty0
  441--450, 1987.

\bibitem[Mundhenk et~al.(2000)Mundhenk, Goldsmith, Lusena, and
  Allender]{mundhenk2000complexity}
Martin Mundhenk, Judy Goldsmith, Christopher Lusena, and Eric Allender.
\newblock Complexity of finite-horizon {M}arkov decision process problems.
\newblock \emph{Journal of the ACM (JACM)}, 47\penalty0 (4):\penalty0 681--720,
  2000.

\bibitem[Vlassis et~al.(2012)Vlassis, Littman, and
  Barber]{vlassis2012computational}
Nikos Vlassis, Michael~L Littman, and David Barber.
\newblock On the computational complexity of stochastic controller optimization
  in {{PO{MDP}}s}.
\newblock \emph{ACM Transactions on Computation Theory (TOCT)}, 4\penalty0
  (4):\penalty0 1--8, 2012.

\bibitem[Mossel and Roch(2005)]{mossel2005learning}
Elchanan Mossel and S{\'e}bastien Roch.
\newblock Learning nonsingular phylogenies and hidden {{M}arkov} models.
\newblock In \emph{Proceedings of the thirty-seventh annual ACM symposium on
  Theory of computing}, pages 366--375, 2005.

\bibitem[Zinkevich et~al.(2007)Zinkevich, Johanson, Bowling, and
  Piccione]{zinkevich2007regret}
Martin Zinkevich, Michael Johanson, Michael Bowling, and Carmelo Piccione.
\newblock Regret minimization in games with incomplete information.
\newblock \emph{Advances in neural information processing systems},
  20:\penalty0 1729--1736, 2007.

\bibitem[Gordon(2007)]{gordon2007no}
Geoffrey~J Gordon.
\newblock No-regret algorithms for online convex programs.
\newblock In \emph{Advances in Neural Information Processing Systems}, pages
  489--496. Citeseer, 2007.

\bibitem[Farina and Sandholm(2021)]{farina2021model}
Gabriele Farina and Tuomas Sandholm.
\newblock Model-free online learning in unknown sequential decision making
  problems and games.
\newblock \emph{arXiv preprint arXiv:2103.04539}, 2021.

\bibitem[Kozuno et~al.(2021)Kozuno, M{\'e}nard, Munos, and
  Valko]{kozuno2021model}
Tadashi Kozuno, Pierre M{\'e}nard, R{\'e}mi Munos, and Michal Valko.
\newblock Model-free learning for two-player zero-sum partially observable
  {{M}arkov} games with perfect recall.
\newblock \emph{arXiv preprint arXiv:2106.06279}, 2021.

\bibitem[Shapley(1953)]{shapley1953stochastic}
Lloyd~S Shapley.
\newblock Stochastic games.
\newblock \emph{Proceedings of the national academy of sciences}, 39\penalty0
  (10):\penalty0 1095--1100, 1953.

\bibitem[Liu et~al.(2022{\natexlab{a}})Liu, Chung, Szepesvari, and
  Jin]{Liu2022when}
Qinghua Liu, Alan Chung, Csaba Szepesvari, and Chi Jin.
\newblock When is partially observable reinforcement learning not scary?
\newblock \emph{arXiv preprint}, 2022{\natexlab{a}}.

\bibitem[Azar et~al.(2017)Azar, Osband, and Munos]{azar2017minimax}
Mohammad~Gheshlaghi Azar, Ian Osband, and R{\'e}mi Munos.
\newblock Minimax regret bounds for reinforcement learning.
\newblock In \emph{International Conference on Machine Learning}, pages
  263--272. PMLR, 2017.

\bibitem[Dann et~al.(2017)Dann, Lattimore, and Brunskill]{dann2017unifying}
Christoph Dann, Tor Lattimore, and Emma Brunskill.
\newblock Unifying {{PAC}} and regret: Uniform {{PAC}} bounds for episodic
  reinforcement learning.
\newblock \emph{Advances in Neural Information Processing Systems}, 30, 2017.

\bibitem[Jin et~al.(2018)Jin, Allen-Zhu, Bubeck, and Jordan]{jin2018q}
Chi Jin, Zeyuan Allen-Zhu, Sebastien Bubeck, and Michael~I Jordan.
\newblock Is {Q}-learning provably efficient?
\newblock \emph{Advances in neural information processing systems}, 31, 2018.

\bibitem[Jin et~al.(2020{\natexlab{a}})Jin, Yang, Wang, and
  Jordan]{jin2020provably}
Chi Jin, Zhuoran Yang, Zhaoran Wang, and Michael~I Jordan.
\newblock Provably efficient reinforcement learning with linear function
  approximation.
\newblock In \emph{Conference on Learning Theory}, pages 2137--2143. PMLR,
  2020{\natexlab{a}}.

\bibitem[Zanette et~al.(2020)Zanette, Lazaric, Kochenderfer, and
  Brunskill]{zanette2020learning}
Andrea Zanette, Alessandro Lazaric, Mykel Kochenderfer, and Emma Brunskill.
\newblock Learning near optimal policies with low inherent {B}ellman error.
\newblock In \emph{International Conference on Machine Learning}, pages
  10978--10989. PMLR, 2020.

\bibitem[Jiang et~al.(2017)Jiang, Krishnamurthy, Agarwal, Langford, and
  Schapire]{jiang2017contextual}
Nan Jiang, Akshay Krishnamurthy, Alekh Agarwal, John Langford, and Robert~E
  Schapire.
\newblock Contextual decision processes with low {{B}ellman} rank are
  {{PAC}}-learnable.
\newblock In \emph{International Conference on Machine Learning}, pages
  1704--1713. PMLR, 2017.

\bibitem[Jin et~al.(2021{\natexlab{a}})Jin, Liu, and
  Miryoosefi]{jin2021bellman}
Chi Jin, Qinghua Liu, and Sobhan Miryoosefi.
\newblock {B}ellman eluder dimension: New rich classes of {{RL}} problems, and
  sample-efficient algorithms.
\newblock \emph{Advances in Neural Information Processing Systems}, 34,
  2021{\natexlab{a}}.

\bibitem[Brafman and Tennenholtz(2002)]{brafman2002r}
Ronen~I Brafman and Moshe Tennenholtz.
\newblock R-max-a general polynomial time algorithm for near-optimal
  reinforcement learning.
\newblock \emph{Journal of Machine Learning Research}, 3\penalty0
  (Oct):\penalty0 213--231, 2002.

\bibitem[Wei et~al.(2017)Wei, Hong, and Lu]{wei2017online}
Chen-Yu Wei, Yi-Te Hong, and Chi-Jen Lu.
\newblock Online reinforcement learning in stochastic games.
\newblock In \emph{Advances in Neural Information Processing Systems}, pages
  4987--4997, 2017.

\bibitem[Bai and Jin(2020)]{bai2020provable}
Yu~Bai and Chi Jin.
\newblock Provable self-play algorithms for competitive reinforcement learning.
\newblock \emph{arXiv preprint arXiv:2002.04017}, 2020.

\bibitem[Liu et~al.(2021)Liu, Yu, Bai, and Jin]{liu2021sharp}
Qinghua Liu, Tiancheng Yu, Yu~Bai, and Chi Jin.
\newblock A sharp analysis of model-based reinforcement learning with
  self-play.
\newblock In \emph{International Conference on Machine Learning}, pages
  7001--7010. PMLR, 2021.

\bibitem[Bai et~al.(2020)Bai, Jin, and Yu]{NEURIPS2020_172ef5a9}
Yu~Bai, Chi Jin, and Tiancheng Yu.
\newblock Near-optimal reinforcement learning with self-play.
\newblock In \emph{Advances in Neural Information Processing Systems}, 2020.

\bibitem[Xie et~al.(2020)Xie, Chen, Wang, and Yang]{xie2020learning}
Qiaomin Xie, Yudong Chen, Zhaoran Wang, and Zhuoran Yang.
\newblock Learning zero-sum simultaneous-move markov games using function
  approximation and correlated equilibrium.
\newblock \emph{arXiv preprint arXiv:2002.07066}, 2020.

\bibitem[Jin et~al.(2021{\natexlab{b}})Jin, Liu, and Yu]{jin2021power}
Chi Jin, Qinghua Liu, and Tiancheng Yu.
\newblock The power of exploiter: Provable multi-agent rl in large state
  spaces.
\newblock \emph{arXiv preprint arXiv:2106.03352}, 2021{\natexlab{b}}.

\bibitem[Jin et~al.(2021{\natexlab{c}})Jin, Liu, Wang, and Yu]{jin2021v}
Chi Jin, Qinghua Liu, Yuanhao Wang, and Tiancheng Yu.
\newblock V-learning--a simple, efficient, decentralized algorithm for
  multiagent {RL}.
\newblock \emph{arXiv preprint arXiv:2110.14555}, 2021{\natexlab{c}}.

\bibitem[Song et~al.(2021)Song, Mei, and Bai]{song2021can}
Ziang Song, Song Mei, and Yu~Bai.
\newblock When can we learn general-sum markov games with a large number of
  players sample-efficiently?
\newblock \emph{arXiv preprint arXiv:2110.04184}, 2021.

\bibitem[Daskalakis et~al.(2022)Daskalakis, Golowich, and
  Zhang]{daskalakis2022complexity}
Constantinos Daskalakis, Noah Golowich, and Kaiqing Zhang.
\newblock The complexity of markov equilibrium in stochastic games.
\newblock \emph{arXiv preprint arXiv:2204.03991}, 2022.

\bibitem[Krishnamurthy et~al.(2016)Krishnamurthy, Agarwal, and
  Langford]{krishnamurthy2016pac}
Akshay Krishnamurthy, Alekh Agarwal, and John Langford.
\newblock {PAC} reinforcement learning with rich observations.
\newblock \emph{Advances in Neural Information Processing Systems}, 29, 2016.

\bibitem[Guo et~al.(2016)Guo, Doroudi, and Brunskill]{guo2016pac}
Zhaohan~Daniel Guo, Shayan Doroudi, and Emma Brunskill.
\newblock A {{PAC} {RL}} algorithm for episodic {{PO{MDP}}s}.
\newblock In \emph{Artificial Intelligence and Statistics}, pages 510--518.
  PMLR, 2016.

\bibitem[Azizzadenesheli et~al.(2016)Azizzadenesheli, Lazaric, and
  Anandkumar]{azizzadenesheli2016reinforcement}
Kamyar Azizzadenesheli, Alessandro Lazaric, and Animashree Anandkumar.
\newblock Reinforcement learning of {PO{MDP}}s using spectral methods.
\newblock In \emph{Conference on Learning Theory}, pages 193--256. PMLR, 2016.

\bibitem[Jin et~al.(2020{\natexlab{b}})Jin, Kakade, Krishnamurthy, and
  Liu]{jin2020sample}
Chi Jin, Sham~M Kakade, Akshay Krishnamurthy, and Qinghua Liu.
\newblock Sample-efficient reinforcement learning of undercomplete
  {{PO{MDP}}s}.
\newblock \emph{NeurIPS}, 2020{\natexlab{b}}.

\bibitem[Farina et~al.(2020)Farina, Kroer, and Sandholm]{farina2020faster}
Gabriele Farina, Christian Kroer, and Tuomas Sandholm.
\newblock Faster game solving via predictive {B}lackwell approachability:
  Connecting regret matching and mirror descent.
\newblock \emph{arXiv preprint arXiv:2007.14358}, 2020.

\bibitem[Oliehoek(2012)]{oliehoek2012decentralized}
Frans~A Oliehoek.
\newblock Decentralized pomdps.
\newblock In \emph{Reinforcement Learning}, pages 471--503. Springer, 2012.

\bibitem[Oliehoek and Amato(2016)]{oliehoek2016concise}
Frans~A Oliehoek and Christopher Amato.
\newblock \emph{A concise introduction to decentralized POMDPs}.
\newblock Springer, 2016.

\bibitem[Nair et~al.(2003)Nair, Tambe, Yokoo, Pynadath, and
  Marsella]{nair2003taming}
Ranjit Nair, Milind Tambe, Makoto Yokoo, David Pynadath, and Stacy Marsella.
\newblock Taming decentralized pomdps: Towards efficient policy computation for
  multiagent settings.
\newblock In \emph{IJCAI}, volume~3, pages 705--711. Citeseer, 2003.

\bibitem[Bernstein et~al.(2005)Bernstein, Hansen, and
  Zilberstein]{bernstein2005bounded}
Daniel~S Bernstein, Eric~A Hansen, and Shlomo Zilberstein.
\newblock Bounded policy iteration for decentralized pomdps.
\newblock In \emph{Proceedings of the nineteenth international joint conference
  on artificial intelligence (IJCAI)}, pages 52--57, 2005.

\bibitem[Oliehoek et~al.(2008)Oliehoek, Spaan, and
  Vlassis]{oliehoek2008optimal}
Frans~A Oliehoek, Matthijs~TJ Spaan, and Nikos Vlassis.
\newblock Optimal and approximate q-value functions for decentralized pomdps.
\newblock \emph{Journal of Artificial Intelligence Research}, 32:\penalty0
  289--353, 2008.

\bibitem[Szer et~al.(2012)Szer, Charpillet, and Zilberstein]{szer2012maa}
Daniel Szer, Fran{\c{c}}ois Charpillet, and Shlomo Zilberstein.
\newblock Maa*: A heuristic search algorithm for solving decentralized pomdps.
\newblock \emph{arXiv preprint arXiv:1207.1359}, 2012.

\bibitem[Dibangoye et~al.()Dibangoye, Amato, Buffet, and
  Charpillet]{dibangoye2016optimally}
Jilles~Steeve Dibangoye, Christopher Amato, Olivier Buffet, and Fran{\c{c}}ois
  Charpillet.
\newblock Optimally solving dec-pomdps as continuous-state mdps.
\newblock \emph{Journal of Artificial Intelligence Research}, 55:\penalty0
  443--497.

\bibitem[Liu et~al.(2017)Liu, Sivakumar, Omidshafiei, Amato, and
  How]{liu2017learning}
Miao Liu, Kavinayan Sivakumar, Shayegan Omidshafiei, Christopher Amato, and
  Jonathan~P How.
\newblock Learning for multi-robot cooperation in partially observable
  stochastic environments with macro-actions.
\newblock In \emph{2017 IEEE/RSJ International Conference on Intelligent Robots
  and Systems (IROS)}, pages 1853--1860. IEEE, 2017.

\bibitem[Amato et~al.(2019)Amato, Konidaris, Kaelbling, and
  How]{amato2019modeling}
Christopher Amato, George Konidaris, Leslie~P Kaelbling, and Jonathan~P How.
\newblock Modeling and planning with macro-actions in decentralized pomdps.
\newblock \emph{Journal of Artificial Intelligence Research}, 64:\penalty0
  817--859, 2019.

\bibitem[Gmytrasiewicz and Doshi(2005)]{gmytrasiewicz2005framework}
Piotr~J Gmytrasiewicz and Prashant Doshi.
\newblock A framework for sequential planning in multi-agent settings.
\newblock \emph{Journal of Artificial Intelligence Research}, 24:\penalty0
  49--79, 2005.

\bibitem[Doshi et~al.(2009)Doshi, Zeng, and Chen]{doshi2009graphical}
Prashant Doshi, Yifeng Zeng, and Qiongyu Chen.
\newblock Graphical models for interactive pomdps: representations and
  solutions.
\newblock \emph{Autonomous Agents and Multi-Agent Systems}, 18\penalty0
  (3):\penalty0 376--416, 2009.

\bibitem[Liu et~al.(2022{\natexlab{b}})Liu, Wang, and Jin]{liu2022learning}
Qinghua Liu, Yuanhao Wang, and Chi Jin.
\newblock Learning markov games with adversarial opponents: Efficient
  algorithms and fundamental limits.
\newblock \emph{arXiv preprint arXiv:2203.06803}, 2022{\natexlab{b}}.

\bibitem[Jakobsen et~al.(2016)Jakobsen, S{\o}rensen, and
  Conitzer]{jakobsen2016timeability}
Sune~K Jakobsen, Troels~B S{\o}rensen, and Vincent Conitzer.
\newblock Timeability of extensive-form games.
\newblock In \emph{Proceedings of the 2016 ACM Conference on Innovations in
  Theoretical Computer Science}, pages 191--199, 2016.

\bibitem[Bai et~al.(2022)Bai, Jin, Mei, and Yu]{bai2022near}
Yu~Bai, Chi Jin, Song Mei, and Tiancheng Yu.
\newblock Near-optimal learning of extensive-form games with imperfect
  information.
\newblock \emph{arXiv preprint arXiv:2202.01752}, 2022.

\end{thebibliography}
\bibliographystyle{unsrtnat}

\newpage

\appendix


\newcommand{\cblue}[1]{{\color{blue}#1}}

\section{On Relation between IIEFGs and Weakly Revealing POMGs}
\label{sec:iiefg}

In this section, we consider imperfect-information extensive-form games with perfect recall, which we call IIEFGs for simplicity. In this section, we will show that any IIEFG (of a polynomial size) is also a weakly revealing POMG (of a polynomial size) that satisfies Assumption \ref{asp:under} with $\alpha=1$. As a result, all the algorithms and the polynomial sample complexity results developed in this paper immediately apply to learning IIEFGs using polynomial samples.


We note that the reverse is not true, we can easily construct a weakly-revealing POMGs of a polynomial size that can not be represented by any IIEFGs with a polynomial size, due to the restriction of tree-structured transition and deterministic transition in IIEFGs. Therefore, polynomial sample complexity for learning IIEFGs does not imply polynomial sample complexity results for learning POMGs. 

\subsection{Representing IIEFGs as 1-Weakly Revealing POMGs}

We first introduce the definition of IIEFGs. There are many equivalent formulations of IIEFGs and here we adopt the formulation used in \cite{kozuno2021model}, which allows a clearer  comparison to POMGs.
\begin{definition}
    An \emph{imperfect information extensive-form game with perfect recall}\footnote{Strictly speaking, we restrict our attention to timeable IIEFGs. We remark that, as argued by \cite{jakobsen2016timeability}, non-timeable IIEFGs can not be implemented in practical systems.} is a POMG$(H, \cS, \{\cA_i\}_{i=1}^n, $\\$ \{\cO_i\}_{i=1}^n; \T,\O,\mu_1; \{r_i\}_{i=1}^n)$ that additionally satisfies the followings: 
    \begin{itemize}
        \item \textbf{Tree-structured transition}: for each $s\in\cS$ and $h\in[H-1]$, there is at most one state-action pair  $(s',\a')\in\cS\times\cA$ such that $\T_h(s\mid s',\a')\neq 0$. In other words, for any $s_h$, there is a unique history sequence $(s_1,\a_1,\ldots,s_{h-1},\a_{h-1})$ that leads to $s_h$.
         
        \item \textbf{Deterministic emission and perfect-recall}: for each $s\in\cS$ and $h\in[H]$, $\|\O_h(\cdot \mid s)\|_0 =1$. That is, no state can emit two different observations. Moreover, for each player $i$ and $x\in\cO_i$, there is a unique history $(o_{i,1},a_{i,1},\ldots,o_{i,h}=x)$ up to $x$ from player $i$'s perspective. This means player $i$ can always retrieve her previous observations and actions solely from her current-step observation. In IIEFGs, the observations are usually referred to as \emph{information sets}.
        
        \item \textbf{Delayed and state-action-dependent reward}: different from our definition of reward in Section \ref{sec:prelim}, now each $r_{i,h}$ is a random function from $\cS\times\cA$ to $[0,1]$, and  the  rewards are  revealed to each learner only at the end of each episode. In other words, player $i$ gets to observe $r_{i,1}^k,\ldots,r_{i,H}^k$ after the $k^{\rm th}$ episode is finished. \footnote{We WLOG consider the delayed reward since almost all algorithms in the IIEFG literature only use information sets (i.e., do not use additional information in the intermediate reward) to make decisions.}
    \end{itemize}
\end{definition}

We now show that any IIEFG can be represented by $1$-weakly-revealing POMGs.

\begin{theorem}\label{thm:iiefg-pomg}
    Any IIEFG$(H, \cS, \{\cA_i\}_{i=1}^n, \{\cO_i\}_{i=1}^n; \T,\O,\mu_1; \{r_i\}_{i=1}^n)$ can be represented as a  POMG with $\prod_i |\cO_i |$ states,  the same action space, the same observation space, stochastic rewards which  depend on  the joint observation and action, and satisfying the single-step weakly revealing condition (Assumption \ref{asp:under}) with $\alpha=1$.
\end{theorem}

Theorem \ref{thm:iiefg-pomg} shows that any IIEFG of a polynomial size can be efficiently represented by $1$-weakly-revealing POMGs with a polynomial size. Here, we consider the number of player $n$ as constant when discussing polynomial versus exponential. 

\begin{proof}[Proof of Theorem \ref{thm:iiefg-pomg}]
We consider an equivalent POMG formulation denoted as $$(H, \cblue{\tilde\cS}, \{\cA_i\}_{i=1}^n, \{\cO_i\}_{i=1}^n; \cblue{\tilde\T},\cblue{\tilde\O},\cblue{\tilde\mu_1}; \cblue{\{\tilde r_i\}_{i=1}^n)}$$ 
where we highlight the modified parts in blue and define them as following:   
\begin{itemize}
    \item \textbf{State and transition.} Notice that the joint observation in IIEFGs always satisfies the Markov property because of the perfect-recall emission structure: 
$$
\P(\o_{h+1}\mid \o_1,\a_1,\ldots,\o_h,\a_h) = \P(\o_{h+1}\mid \o_h,\a_h).
$$
Therefore we can view the original joint observation space as the new state space $\tilde \cS:= \prod_i \cO_i$ and define the transition as 
$$
\tilde \T_h ( \tilde s ' \mid \tilde s, \a) := \P(\o_{h+1} = \tilde s' \mid \o_h = \tilde s, \a_h=\a).
$$
And the initial distribution is defined as $\tilde \mu_1 := \P(\o_1=\cdot \mid s_1 \sim \mu_1)$.

\item \textbf{Emission.} We define the emission so that player $i$ always observes $[\tilde s_h]_i$ (the $i^{\rm th}$ entry in $\tilde s_h$) with probability $1$ at step $h$. Formally for all $h\in[H]$ and $(\o,\tilde s)\in \cO\times\tilde\cS$
$$
\tilde \O_h( \o \mid  \tilde s )=  \mathbf{1}(\o = \tilde s).
$$
Clearly, in this case the joint emission is identity and therefore  satisfies the single-step weakly revealing condition (Assumption \ref{asp:under}) with $\alpha=1$.

\item \textbf{Reward.} As for the reward function, we let $\tilde r_{i,h} :=0$ for $h \le H-1$, and  define $\tilde r_{i,H}(\o_{H},\a_H)$
to be a random variable taking value $\sum_{h=1}^H r_{i,h}( s_h,\a_h)$ with $ s_{1:H}$  sampled from 
$$
\P(s_{1:H} =\cdot \mid   \o_1,\a_1,\ldots,\o_{H},\a_{H}) 
= \P(s_{1:H} =\cdot \mid   \o_{H},\a_H). 
$$
Therefore, the reward $\tilde r_{i,H}$ is a random function of the joint observation and action $(\o_{H},\a_H)$ at step $H$. 
\end{itemize}
 It is direct to see any policy induces the same distribution over $\o_1,\a_1,\ldots,\o_{H},\a_{H}$ and enjoys the same value in this new formulation as in the original IIEFG. 
 As a result, any algorithms designed for weakly revealing POMGs also apply to learning IIEFGs.

\paragraph{Stochastic reward depending on  the joint observation and action}
Recall when defining POMGs in Section \ref{sec:prelim}, we let the reward to be a deterministic function of individual observations. Nonetheless, one can easily verify all our results in this paper still hold without non-trivial modifications when the reward functions are  stochastic and depend on the joint observation and action. 
As a result, we conclude that any IIEFG with $O$ observations, $S$ latent states and $A$ actions can be represented as a  $1$-weakly revealing POMG with $O$ observations, $O$ latent states and $A$ actions, to which all our algorithms and theoretical guarantees directly apply.  
\end{proof}

\paragraph{Regarding the curse of multi-player} Note that all the sample complexity results proved in this paper scale exponentially with respect to $n$, the number of players. Therefore, they suffer from the curse of multi-players when $n$ is large. In particular, when specializing these results to the setting of IIEFGs, we obtain sample complexity scaling with $\prod_{i\in[n]} | \cO_i |$ instead of $\sum_{i\in[n]} | \cO_i |$ where the latter is achievable by algorithms specially designed for learning IIEFGs \citep[e.g.,][]{bai2022near}. Nonetheless, we observe that the $1$-weakly revealing POMG presentation of IIEFGs derived in Theorem \ref{thm:iiefg-pomg} possesses additional benign structures: the state space is factored and the emission is identity, which could potentially be exploited to overcome the curse of dimensionality with sharper analysis or different algorithm design (e.g., incorporate the idea of V-learning style algorithms \cite{jin2021v,song2021can,daskalakis2022complexity}).

\subsection{On Inefficiency of Representing POMGs using IIEFGs}

We prove two theorems for representing POMGs using IIEFGs: 
\begin{itemize}
\item First we show POMGs can be represented by IIEFGs with an exponentially large size. (Exponential large model is prohibitive in practice, more relevant question is for the polynomial size).
\item Then we prove a lower bound showing that there exists weakly revealing POMGs of constant size, which can not be represented by any IIEFGs with a polynomial size. 
This implies  IIEFGs can not efficiently represent POMGs and polynomial results for learning IIEFGs can not translate into efficiency guarantees for learning POMGs.
\end{itemize}

\begin{theorem}\label{thm:pomg-iiefg}
    A POMG$(H, \cS, \{\cA_i\}_{i=1}^n, \{\cO_i\}_{i=1}^n; \T,\O,\mu_1; \{r_i\}_{i=1}^n)$ can be represented as an  IIEFG with $(\prod_i |\cO_i |  |\cA_i|)^H$ states,  the same action space,   $(|\cO_i |  |\cA_i|)^H$  observations for each player $i\in[n]$.
     \end{theorem}
   \begin{proof}[Proof of Theorem \ref{thm:pomg-iiefg}]
   We consider an equivalent IIEFG formulation denoted as $$(H, \cblue{\tilde\cS}, \{\cA_i\}_{i=1}^n, \{\cO_i\}_{i=1}^n; \cblue{\tilde\T},\cblue{\tilde\O},\cblue{\tilde\mu_1}; \cblue{\{\tilde r_i\}_{i=1}^n)}$$ 
where we highlight the modified parts in blue and define them as following:   
\begin{itemize}
    \item \textbf{State and transition.}
    We view the entire interaction history  as the state of IIEFG, that is, $\tilde s_h = (\o_1,\a_1,\ldots,\o_h)$. Under such choice of latent state, the transition is clearly tree structured and satisfies: for any $\tilde s_h = (\o_1,\a_1,\ldots,\o_h)$ and 
    $\tilde s_{h+1} = (\o_1',\a_1',\ldots,\o_{h+1}')$ 
    $$
\tilde \T_h (\tilde s_{h+1} \mid \tilde s_{h}, \a_h) = \P(\o_{h+1}' \mid  (\o',\a')_{1:h}) \times \mathbf{1}(   (\o,\a)_{1:h}=  (\o',\a')_{1:h}).
$$

\item \textbf{Emission.} We define the emission so that player $i$ always observes $[\tilde s_h]_i$ (the $i^{\rm th}$ entry in $\tilde s_h$) with probability $1$ at step $h$. Formally for all $h\in[H]$, if the environment is at state $\tilde s_h = (\o_1,\a_1,\ldots,\o_h)$, then each player $i$ will observe $(o_{i,1},a_{i,1},\ldots,o_{i,h})$ with probability $1$. 
By the definition of state and transition, $[\tilde s_h]_i$ is exactly equal to the interaction history of player $i$. 
Therefore, such emission structure satisfies the perfect-recall condition. 

\item \textbf{Reward.} As for the reward function, we let $\tilde r_{i,h} :=0$ for $h \le H-1$. 
At step $H$, for any  $\tilde s_H = (\o_1,\a_1,\ldots,\o_H)$ and $\a_H$, we define 
$$
\tilde r_{i,H}(\tilde s_H, \a_H) = \sum_{h=1}^H r_{i,h}( o_{i,h},a_{i,h}).
$$ 
\end{itemize}
   \end{proof}

\begin{theorem}\label{thm:pomg-iiefg-ineff}
   There exists an $\Omega(1)$-weakly revealing POMG of size $\cO(1)$, which is not equivalent to any perfect-recall IIEFG with $ \max_{i\in[n]} |\cO_i|  \le 4^{H-1}$.
     \end{theorem}
  \begin{proof}[Proof of Theorem \ref{thm:pomg-iiefg-ineff}]
  It suffices to prove the above theorem for the single-agent case, i.e., POMDPs. 
 Consider a POMDP with $2$ states, $2$ actions and $2$ observations. The emission and transition is defined as 
 $$
\T_{h,i} = \begin{pmatrix}
\alpha_{h,i} & 1- \alpha_{h,i}\\
1- \alpha_{h,i} &  \alpha_{h,i}
\end{pmatrix}  \quad \mbox{and}  \quad 
\O_{h} = \begin{pmatrix}
\beta_{h} & 1- \beta_{h}\\
1- \beta_{h} &  \beta_{h}
\end{pmatrix}, $$
where $\{\alpha_{h,i} \}_{(h,i)\in[H]\times[2]}$ and $\{\beta_{h} \}_{h\in[H]}$ are i.i.d. sampled from $[0,1/2]$.
To represent the above POMDP as IIEFG with perfect recall, the size of the observation space must be at least 
$4^{H-1}$ since there are $4^{H-1}$ different possible  trajectories of form $(o_1,a_1,\ldots,o_{H-1},a_{H-1})$.
    \end{proof}

\section{Notations} We first introduce some notations that will be frequently used in the remainder of appendix.
\begin{itemize} 
    \item We will use $\mu\in\Pidet$ to refer to a \emph{deterministic} joint policy, and use $\mu_i\in\Pidet_i$ to refer to a \emph{deterministic} policy of player $i$. 
    
    \item Since each stochastic joint policy $\pi\in\Pi$ is equivalent to a distribution over all the deterministic joint polices $\Pidet$, with slight abuse of notation, we denote by $\mu\sim\pi$ the process of sampling a deterministic joint policy $\mu$ from the policy distribution specified by $\pi$. We can similarly define $\mu_i\sim\pi_i$ for any stochastic policy  $\pi_=i$ of player $i$.
    
    \item Given a policy $\pi$ and a POMG model $\theta$, denote by $\P^\pi_\theta$ the distribution over trajectories (i.e., $\tau_H$) produced by executing policy $\pi$ in a POMG parameterized by $\theta$. Since the reward per trajectory is bounded by $H$, we always have 
    $$
V_i^\pi(\theta)-V_i^\pi(\hat\theta) \le H \| \P_\theta^\pi - \P_{\hat\theta}^\pi\|
    $$
    for any policy $\pi$, POMG models $\theta,\hat\theta$, and player $i$.
    
    \item Denote by $\theta^\star$ the parameters of the groundtruth POMG model we are interacting with.
\end{itemize}

\section{Proofs for the Self-play Setting}

\subsection{Proof of Theorem \ref{thm:under}}
\label{subsec:proof-under}

In this section, we prove  Theorem \ref{thm:under} with a specific  polynomial dependency as stated in the following theorem.
\begin{theorem}\label{thm:under-rate}
    \emph{(Regret of \algName)}
        Under Assumption \ref{asp:under},
        there exists an absolute constant $c$ such that for any $\delta\in(0,1]$ and $K\in\N$, Algorithm \ref{alg:under} with $\beta = c\left(H(S^2A+SO)\log(SAOHK)+\log(K/\delta)\right)$ and \textsc{Equilibrium} being one of $\{\text{Nash, CCE, CE}\}$ satisfies (respectively) that  
        with probability at least $1-\delta$, 
        $$\textstyle \reg_{\{\rm Nash,CCE,CE\}}(k) \le
        \tilde{\cO}\left( \frac{S^{2}AO}{\alpha^2} \sqrt{k(S^2A+SO)} \times \poly(H)\right)\qquad \text{ for all }k\in[K].$$
\end{theorem}

The proof consists of three steps:
\begin{enumerate}
    \item  First we rewrite Algorithm \ref{alg:under} in an equivalent form that is perfectly compatible with the analysis in \cite{Liu2022when}.
    \item After that we can directly  import the theoretical guarantees from \cite{Liu2022when} and obtain a sublinear upper bound for the cumulative error of density estimation. 
    \item Finally, we combine the game-theoretic analysis tailored for POMGs with the density estimation guarantee derived in the second step, which gives the desired sublinear game-theoretic regret. 
\end{enumerate}
\subsubsection{Step 1}\label{step:1}
To begin with, we make the following observations about Algorithm \ref{alg:under}:
\begin{itemize}
    \item The sampling procedure in each episode $k$ is equivalent to: first sample a \emph{deterministic} joint policy $\mu^k$ from $\pi^k$ and then execute $\mu^k$ to collect a trajectory $\tau^k$. 
    \item In constructing the confidence set $\cB^k$, we can replace $\pi^k$ with $\mu^k$ without making any difference, because the dependency of the log-likelihood function on policy $\pi$ are equal on both sides of the inequality in $\cB^k$ and thus they cancel with each other. Formally, for any  $\hat\theta,\theta'\in\Theta$, we have
    \begin{align*}
    &\sum_{t=1}^{k} \left(\log \P_{{\hat\theta}}^{\pi^t} (\tau^t)
    -  \log \P^{\pi^t}_{{\theta'}}(\tau^t) \right) \\
     = &\sum_{t=1}^{k} \left(\log \P_{{\hat\theta}} (\o_{1:H}^t\mid \a_{1:H}^t)
    -  \log \P_{{\theta'}}(\o_{1:H}^t\mid \a_{1:H}^t) \right) 
    =  \sum_{t=1}^{k} \left(\log \P_{{\hat\theta}}^{\mu^t} (\tau^t)
    -  \log \P^{\mu^t}_{{\theta'}}(\tau^t) \right).
    \end{align*}    
\end{itemize}
Based on the above two observations, Algorithm \ref{alg:under} can be \emph{equivalently} written in the form of Algorithm \ref{alg:under-equiv} where we highlight the modified parts in blue.

\begin{remark}
    The technical reason for rewriting Algorithm \ref{alg:under} in the form of Algorithm \ref{alg:under-equiv} is that in the optimistic equilibrium subroutine (Subroutine \ref{subalg}) we utilize the optimistic value estimate for each \emph{deterministic} joint policy to construct the optimistic normal-form game and compute the optimistic game-theoretic equilibria. 
    As a result, in order to control the cumulative regret due to over-optimism, we need  guarantees on the accuracy of optimistic value estimates for \emph{deterministic} joint policies. This is why we want to explicitly insert the ``dummy'' deterministic policy $\mu^k$ in each episode.
\end{remark}
\begin{algorithm}[t]
    \caption{\algName}
 \begin{algorithmic}[1]\label{alg:under-equiv}
\STATE \textbf{Initialize:} $\cB^1 = \{ \hat \theta\in\Theta:~\min_h\sigma_{S}(\hat\O_h)\ge \alpha\}$, $\cD=\{\}$ 
\FOR{$k=1,\ldots,K$}
\STATE compute $\pi^k=$\textsf{Optimistic\_Equilibrium}$(\cB^k)$  
    \STATE  \cblue{sample a deterministic joint plicy $\mu^k$ from $\pi^k$, then follow $\mu^k$ to collect a trajectory $\tau^k$ }
    \STATE add $(\cblue{\mu^k},\tau^k)$ into $\cD$  and update 
    $$
    \cB^{k+1} = \bigg\{\hat\theta \in \Theta: \sum_{(\pi,\tau)\in\cD} \log \P_{{\hat\theta}}^{\pi} (\tau)
    \ge \max_{ \theta' \in\Theta} \sum_{(\pi,\tau)\in\cD} \log \P^{\pi}_{{\theta'}}(\tau) -\beta \bigg\} \bigcap \cB^1$$
    \ENDFOR
 \end{algorithmic}
 \end{algorithm}

 \subsubsection{Step 2}\label{step:2}
Now we can directly instantiate  the analysis of optimistic MLE (Appendix E in \cite{Liu2022when}) on Algorithm \ref{alg:under-equiv}, which gives the following theoretical guarantee:
\begin{theorem} \emph{(\cite{Liu2022when})}\label{thm:pomdp-under}
Under Assumption \ref{asp:under} and the same choice of $\beta$ as in Theorem \ref{thm:under}, with probability at least $1-\delta$, Algorithm \ref{alg:under-equiv} satisfies that for \textbf{all} $k\in[K]$ and \textbf{all} $\theta^1\in\cB^1$,\dots,$\theta^K\in\cB^K$
\begin{itemize}
    \item $\theta^\star\in\cB^k$ ,
    \item $\sum_{t=1}^k  \| \P^{\mu^t}_{\theta^t} -\P^{\mu^t}_{\theta^\star} \|_{1} \le 
    \tilde{\cO}\left( \frac{S^{2}AO}{\alpha^2} \sqrt{k(S^2A+SO)} \times \poly(H)\right)
    $.
\end{itemize}
\end{theorem}
We comment that there are two differences between the optimistic MLE algorithm in \cite{Liu2022when} and the Algorithm \ref{alg:under-equiv} here: (i) the former one is designed for single-player POMGs, i.e., POMDPs while the latter one is for multi-player POMGs; (ii)  $\mu^t$ is computed using different criteria. Nonetheless, we can still reuse their theoretical guarantees proved in their Appendix E without making any change because: (i) in the self-play setting, multi-player POMGs can be viewed as POMDPs with a single meta-player whose action space is of cardinality $A=A_1\times\cdots\times A_n$ and observation space is of cardinality $O=O_1\times\cdots\times O_n$; 
(ii) when proving the second statement in Theorem \ref{thm:pomdp-under}, \cite{Liu2022when} only use the fact that $\tau^t$ is sampled from $\mu^t$ but allow both $\mu^t$ and $\theta^t\in\cB^t$ to be arbitrarily chosen. (The only place \cite{Liu2022when} need to use how $\mu^t$ and $\theta^t$ is computed is in  relating the regret to $\sum_{t=1}^k  \| \P^{\mu^t}_{\theta^t} -\P^{\mu^t}_{\theta^\star} \|_{1}$, which has nothing to do with the proof of Theorem \ref{thm:pomdp-under}.)

\subsubsection{Step 3}
Now let us prove Theorem \ref{thm:under-rate} conditioning on the two relations stated in Theorem \ref{thm:pomdp-under} being true.
To proceed, we define
$$
\up{V}^{k,\mu}_i = \max_{\hat \theta \in\cB^k} V_i^\mu(\hat\theta) \quad \text{ for any } (\mu,k,i)\in\Pidet\times[K]\times[n].
$$
Note that conditioning on the first relation in Theorem \ref{thm:pomdp-under}, we always have $\up{V}^{k,\mu}_i  \ge V_i^\mu$ for all $(\mu,k,i)\in\Pidet\times[K]\times[n]$ because by definition $V_i^\mu = V_i^\mu(\theta^\star)$.
\paragraph{Nash equilibrium} 
When we choose \textsc{Equilibrium} in Subroutine \ref{subalg} to be  Nash equilibrium, 
by the definition of Nash-regret,
\begin{equation}\label{eq:proof-1}
\begin{aligned}
    \reg_{\rm Nash}(K) &= \sum_{k} \max_i \left(\max_{\mu_i\in\Pidet_i} V_i^{\mu_i\times\pi_{-i}^k}- V_i^{\pi^k}\right)\\
    & =\sum_{k} \max_i \left(\max_{\mu_i\in\Pidet_i} \E_{\mu_{-i}\sim\pi_{-i}^k} \left[V_i^{\mu_i\times\mu_{-i}}\right]- V_i^{\pi^k}\right)\\
    & \le \sum_{k} \max_i \left(\max_{\mu_i\in\Pidet_i} \E_{\mu_{-i}\sim\pi_{-i}^k} \left[\up{V}_i^{k,\mu_i\times\mu_{-i}}\right]- V_i^{\pi^k}\right)\\
    & = \sum_{k} \max_i \left( \E_{\mu\sim\pi^k} \left[\up{V}_i^{k,\mu}\right]- \E_{\mu\sim\pi^k} \left[V_i^{\mu}\right]\right),
\end{aligned}
\end{equation}
where the final equality uses the fact that $\pi^k$ is a Nash equilibrium of the normal-form game defined by $(\up{V}_1^k,\ldots,\up{V}_n^k)$ as described in Subroutine \ref{subalg}.
By Jensen's inequality and Azuma-Hoeffding inequality, 
\begin{align*}
    &\quad \sum_{k} \max_i \left( \E_{\mu\sim\pi^k} \left[\up{V}_i^{k,\mu}\right]- \E_{\mu\sim\pi^k} \left[V_i^{\mu}\right]\right)\\
    &  \le \sum_{k}  \E_{\mu\sim\pi^k} \left[\max_i \left( \up{V}_i^{k,\mu}- V_i^{\mu}\right)\right]\\
    &  \le \sum_{k}  \max_i \left( \up{V}_i^{k,\mu^k}- V_i^{\mu^k}\right)+ \tilde{\cO}(H\sqrt{K})\\
    & = \sum_{k}  \max_i \left( \max_{\hat\theta\in\cB^k} V_i^{\mu^k}(\hat\theta)- V_i^{\mu^k}\right)+ \tilde{\cO}(H\sqrt{K}) \\
    & \le H\sum_{k}    \max_{\hat\theta\in\cB^k} \left\|\P^{\mu^k}_{\hat\theta}- \P^{\mu^k}_{\theta^\star} \right\|_1+ \tilde{\cO}(H\sqrt{K}),
\end{align*}
where the last equality uses the definition of $\up{V}^k$ and the last inequality uses the fact that the reward is an $H$-bounded function of the trajectory. Finally, we complete the proof by using the second relation in Theorem \ref{thm:pomdp-under}, which upper bounds $\sum_{k}    \max_{\hat\theta\in\cB^k} \left\|\P^{\mu^k}_{\hat\theta}- \P^{\mu^k}_{\theta^\star} \right\|_1$ by
$ \tilde{\cO}\left( \frac{S^{2}AO}{\alpha^2} \sqrt{K(S^2A+SO)} \times \poly(H)\right)$. 

\paragraph{Coarse correlated equilibrium} 
When we choose \textsc{Equilibrium} in Subroutine \ref{subalg} to be  CCE, the proof is exactly the same as for Nash equilibrium, except that the last equality in Equation \eqref{eq:proof-1} becomes ``no larger than'' by the definition of CCE.

\paragraph{Correlated equilibrium} When we choose \textsc{Equilibrium} in Subroutine \ref{subalg} to be  CE, 
by the definition of CE-regret,
\begin{equation}\label{eq:proof-2}
\begin{aligned}
    \reg_{\rm CE}(K) &= \sum_{k} \max_i \left(\max_{\phi_i} V_i^{(\phi_i\diamond\pi_i^k)\odot \pi_{-i}^k}- V_i^{\pi^k}\right)\\
    & =\sum_{k} \max_i \left(\max_{\phi_i}  \E_{\mu\sim\pi^k} \left[V_i^{(\phi_i\diamond\mu_i)\times \mu_{-i}}\right]- V_i^{\pi^k}\right)\\
    & \le \sum_{k} \max_i \left(\max_{\phi_i}  \E_{\mu\sim\pi^k} \left[\up{V}_i^{k,(\phi_i\diamond\mu_i)\times \mu_{-i}}\right]- V_i^{\pi^k}\right)\\
    & = \sum_{k} \max_i \left( \E_{\mu\sim\pi^k} \left[\up{V}_i^{k,\mu}\right]- \E_{\mu\sim\pi^k} \left[V_i^{\mu}\right]\right),
\end{aligned}
\end{equation}
where the second equality uses the definition of strategy modification, and the final equality uses the fact that $\pi^k$ is a CE of the normal-form game defined by $(\up{V}_1^k,\ldots,\up{V}_n^k)$ as described in Subroutine \ref{subalg}.
The remaining steps are the same as of the proof for Nash-regret.

\subsection{Proof of Theorem \ref{thm:over}}

In this section, we prove  Theorem \ref{thm:over} with a specific  polynomial dependency as stated in the following theorem.
\begin{theorem}
    \emph{(Total suboptimality of \algNameMulti)}
    \label{thm:over-rate}
    Under Assumption \ref{asp:over},
	there exists an absolute constant $c$ such that for any $\delta\in(0,1]$ and $K\in\N$, Algorithm \ref{alg:over} with 
	$$\beta = c\left(H(S^2A+SO)\log(SAOHK)+\log(K/\delta)\right)$$ and \textsc{Equilibrium} being one of $\{\text{Nash, CCE, CE}\}$ satisfies (respectively) that  
	with probability at least $1-\delta$, 
	$$\textstyle \reg_{\{\rm Nash,CCE,CE\}}(k) \le  \tilde{\mathcal{O}}\left( \frac{S^2 A^{3m-2} }{\alpha^2}\sqrt{k(S^2A+SO)}\times\poly(H)\right)  \qquad \text{ for all }k\in[K],$$
    where the regret is computed for policy $\pi^1,\ldots,\pi^k$.
\end{theorem}

The proof of Theorem \ref{thm:over-rate} follows basically the same arguments as in the undercomplete setting, except that we replace Algorithm \ref{alg:under-equiv} with Algorithm \ref{alg:over-equiv} \footnote{We remark that in each episode $k$ of Algorithm \ref{alg:over}, we  sample the random seed $\omega$ used in $\pi^k$ \emph{only once} and then combine $\pi^k(\omega,\cdot)$ with random actions starting from different in-episode steps to collect multiple trajectories (Line \ref{alg2-4}-\ref{alg2-5} in Algorithm \ref{alg:over}). 
Therefore, Algorithm \ref{alg:over-equiv} and Algorithm \ref{alg:over} are \emph{equivalent} for the same reasons  as explained in Section \ref{step:1}.} and Theorem \ref{thm:pomdp-under} with Theorem \ref{thm:pomdp-over} in the first two steps. 
And the third step is exactly the same.
To avoid noninformative repetitive arguments, here we only state Algorithm \ref{alg:over-equiv} and Theorem \ref{thm:pomdp-over}, while one can directly verify all the proofs in Section \ref{subsec:proof-under} still hold after we make the aforementioned replacements.

\begin{algorithm}[t]
    \caption{\algNameMulti}
 \begin{algorithmic}[1]\label{alg:over-equiv}
\STATE \textbf{Initialize:} $\cB^1 = \{ \hat \theta\in\Theta:~\min_h\sigma_{S}(\hat\M_h)\ge \alpha\}$, $\cD=\{\}$ 
\FOR{$k=1,\ldots,K$}
\STATE compute $\pi^k=$\textsf{Optimistic\_Equilibrium}$(\cB^k)$ and \cblue{sample $\mu^k$ from $\pi^k$ }
    \FOR{$h=0,\ldots,H-m$} 
    \STATE execute policy $\cblue{\mu^k_{1:h}}\circ \text{uniform}(\mathcal{\cA})$  to collect a trajectory $\tau^{k,h}$ \\
    then add  $(\cblue{\mu^k_{1:h}}\circ \text{uniform}(\mathcal{\cA}),\tau^{k,h})$ into $\cD$ 
   \ENDFOR
    \STATE update 
    $$
    \cB^{k+1} = \bigg\{\hat\theta \in \Theta: \sum_{(\pi,\tau)\in\cD} \log \P_{{\hat\theta}}^{\pi} (\tau)
    \ge \max_{ \theta' \in\Theta} \sum_{(\pi,\tau)\in\cD} \log \P^{\pi}_{{\theta'}}(\tau) -\beta \bigg\} \bigcap \cB^1
    $$ 
    \ENDFOR
 \end{algorithmic}
 \end{algorithm}

 \begin{theorem} \emph{(\cite{Liu2022when})}\label{thm:pomdp-over}
    Under Assumption \ref{asp:over} and the same choice of $\beta$ as in Theorem \ref{thm:over}, with probability at least $1-\delta$, Algorithm \ref{alg:over-equiv} satisfies that for \textbf{all} $k\in[K]$ and \textbf{all} $\theta^1\in\cB^1$,\dots,$\theta^K\in\cB^K$
    \begin{itemize}
        \item $\theta^\star\in\cB^k$ ,
        \item $\sum_{t=1}^k  \| \P^{\mu^t}_{\theta^t} -\P^{\mu^t}_{\theta^\star} \|_{1} \le \tilde{\mathcal{O}}\left( \frac{S^2 A^{3m-2} }{\alpha^2}\sqrt{k(S^2A+SO)}\times\poly(H)\right) $.
    \end{itemize}
    \end{theorem}

 We remark that Theorem \ref{thm:pomdp-over} follows directly from instantiating  the analysis of multi-step optimistic MLE (Appendix F in \cite{Liu2022when})  on Algorithm \ref{alg:over-equiv}.

\section{Proofs for Playing against Adversarial Opponents}

\subsection{Proof of Theorem \ref{thm:adv}}

In this section, we prove  Theorem \ref{thm:adv} with a specific  polynomial dependency as stated in the following theorem.
\begin{theorem}\emph{(Regret of \algAdv)}\label{thm:adv-rate}
    Under Assumption \ref{asp:under},
	there exists an absolute constant $c$ such that for any $\delta\in(0,1]$ and $K\in\N$, Algorithm \ref{alg:adv} with $\beta = c\left(H(S^2A+SO)\log(SAOHK)+\log(K/\delta)\right)$ satisfies  that  
	with probability at least $1-\delta$, 
	$$\textstyle
    \sum_{t=1}^k \left(\max_{\tilde \pi_1}\min_{\tilde \pi_2} V_1^{\tilde\pi_1 \times \tilde\pi_2}- V_{1}^{\pi^t}\right) 
    \le \tilde{\cO}\left( \frac{S^{2}AO}{\alpha^2} \sqrt{k(S^2A+SO)} \times \poly(H)\right)\quad \text{for all $k\in[K]$}.
    $$
\end{theorem}

To begin with, for the same reasons as explained in Section \ref{step:2}, we can directly instantiate the guarantees for optimistic MLE (Appendix E in \citep{Liu2022when}) on Algorithm \ref{alg:adv} and obtain:
\begin{theorem} \emph{(\cite{Liu2022when})}\label{thm:pomdp-under-adv}
    Under Assumption \ref{asp:under} and the same choice of $\beta$ as in Theorem \ref{thm:adv}, with probability at least $1-\delta$, Algorithm \ref{alg:adv} satisfies that for \textbf{all} $k\in[K]$ and \textbf{all} $\theta^1\in\cB^1$,\dots,$\theta^K\in\cB^K$
    \begin{itemize}
        \item $\theta^\star\in\cB^k$ ,
        \item $\sum_{t=1}^k  \| \P^{\pi^t}_{\theta^t} -\P^{\pi^t}_{\theta^\star} \|_{1} \le   \tilde{\cO}\left( \frac{S^{2}AO}{\alpha^2} \sqrt{k(S^2A+SO)} \times \poly(H)\right)$.
    \end{itemize}
    \end{theorem}
Now conditioning on the two relations in Theorem \ref{thm:pomdp-under-adv} being true, we have 
\begin{align*}
&\sum_{k}\left( \max_{\hat\pi_1} \min_{\hat\pi_{-1}} V_1^{\hat\pi_1\times\hat\pi_{-1}}- V_1^{\pi_1^k\times \pi_{-1}^k}\right)\\
= & \sum_{k}\left( \max_{\hat\pi_1} \min_{\hat\pi_{-1}} V_1^{\hat\pi_1\times\hat\pi_{-1}}(\theta^\star)- 
\max_{\hat\theta\in\cB^k}\max_{\hat\pi_1} \min_{\hat\pi_{-1}} V_1^{\hat\pi_1\times\hat\pi_{-1}}(\hat\theta)\right)\\
&+ 
\sum_{k}\left( \max_{\hat\theta\in\cB^k}\max_{\hat\pi_1} \min_{\hat\pi_{-1}} V_1^{\hat\pi_1\times\hat\pi_{-1}}(\hat\theta)-
V_1^{\pi_1^k\times \pi_{-1}^k}(\theta^\star)\right)\\
\text{$\theta^\star\in\cB^k$}\quad \le &  \sum_{k}\left( \max_{\hat\theta\in\cB^k}\max_{\hat\pi_1} \min_{\hat\pi_{-1}} V_1^{\hat\pi_1\times\hat\pi_{-1}}(\hat\theta)-
V_1^{\pi_1^k\times \pi_{-1}^k}(\theta^\star)\right)\\
\text{by the definition of $\pi^k_1$}\quad
= &  \sum_{k}\left( \max_{\hat\theta\in\cB^k} \min_{\hat\pi_{-1}} V_1^{\pi_1^k\times\hat\pi_{-1}}(\hat\theta)-
V_1^{\pi_1^k\times \pi_{-1}^k}(\theta^\star)\right)\\
\le &  \sum_{k}\left( \max_{\hat\theta\in\cB^k}  V_1^{\pi_1^k\times\pi_{-1}^k}(\hat\theta)-
V_1^{\pi_1^k\times \pi_{-1}^k}(\theta^\star)\right)\\
\text{reward per episode $\in[0,H]$}\quad
\le & H\sum_{k}  \max_{\hat\theta\in\cB^k}  \left\| \P_{\hat\theta}^{\pi^k} - \P_{\theta^\star}^{\pi^k} \right\|_1 \\
\text{Theorem \ref{thm:pomdp-under-adv}}\quad\le &  \tilde{\cO}\left( \frac{S^{2}AO}{\alpha^2} \sqrt{k(S^2A+SO)} \times \poly(H)\right).
\end{align*}

\subsection{Proof of Theorem \ref{thm:lowerbound}}

\begin{figure}[H]
    \center
    \includegraphics[width=0.8\textwidth]{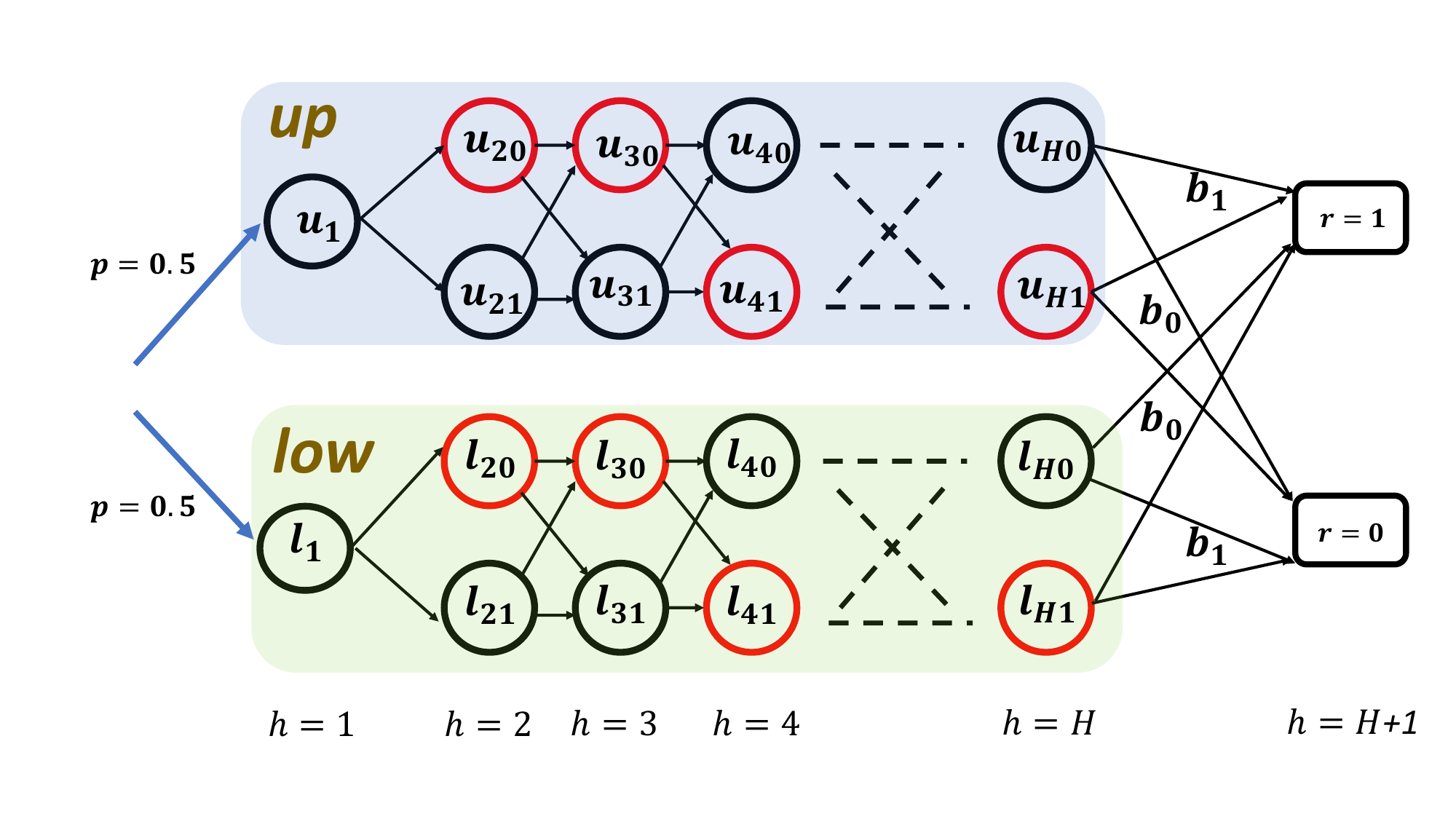}
    \caption{hard instance for Theorem \ref{thm:lowerbound}.}\label{figure}
\end{figure}

The hard instance is best illustrated by Figure \ref{figure} where we sketch the transition dynamics of the POMG. Below we elaborate the construction based on Figure \ref{figure}.
\begin{itemize}
    \item \textbf{States and actions.} Each circle and rectangle in Figure \ref{figure} represents a state. Each player has two  actions denoted by $\cA=\{a_0,a_1\}$ and $\cB=\{b_0,b_1\}$ respectively.
    \item \textbf{Observations.} The max-player can always directly observe the current latent state. The min-player observes the same dummy observation $o_{\rm null}$ in any black circle while directly observes the current state in any red circle and any rectangle. It is easy to verify by definition  $\min_h\sigma_{\min}(\O_h)=1$ because the emission structure is a bijection between the joint observation space and  the latent state space.
    \item \textbf{Reward.} Only the upper rectangle emits an observation containing reward $1$ for the max-player (thus reward $-1$ for the min-player).  
    All other states emit observations with zero reward. 
    \item \textbf{Transitions.} At the beginning of each episode, the environment starts from $u_1$ or $l_1$ uniformly at random. The transition dynamics from step $1$ to step $H$ only depend on the actions of the max-player while in the final step ($H\rightarrow H+1$) the transitions are determined by the min-player. Formally, 
    \begin{itemize}
        \item When the environment is in the \textbf{upper half} of the POMG: for each step $h\in[H-1]$, the environment will transition to $u_{h+1,0}$ if the max-player takes action $a_0$ and transition to $u_{h+1,1}$ if the max-player takes action $a_1$. At step $H$, the agent will transition to the upper rectangle if the min-player takes action $b_1$ and transition to the lower one if the min-player picks $b_0$.
        \item When the environment is in the \textbf{lower half} of the POMG: for each step $h\in[H-1]$, the environment will transition to $l_{h+1,0}$ if the max-player takes action $a_0$ and transition to $l_{h+1,1}$ if the max-player takes action $a_1$. At step $H$, the agent will transition to the upper rectangle if the min-player takes action $b_0$ and transition to the lower one if the min-player picks $b_1$.
    \end{itemize}
\end{itemize}

\paragraph{Min-player's optimal strategy.} It is direct to see the min-player's optimal strategy is to take action $b_0$ in the upper half of the POMG and action $b_1$ in the lower half, at step $H$. 
This stratety will lead to zero-reward for the max-player. 
However, implementing this strategy requires the min-player to infer which half the environment is in from her observations, which is possible only when the environment has visited some red circles in the first $H$ steps. This is because the min-player directly observes the current state in red circles while observes the same observation $o_{\rm null}$ in all black circles.

\paragraph{Max-player's optimal strategy.} 
To prevent the min-player from discovering which half the environment currently lies in, the max-player's optimal strategy is to avoid visiting any red circles. 

\paragraph{Hardness.}
However, hardness happens if (a) the max-player cannot access the observations of the min-player and (b) for each $h\in\{2,\ldots,H\}$, we uniformly at random pick one of $\{u_{h0},u_{h1}\}$ and one of $\{l_{h0},l_{h1}\}$ to be red circles, and set the remaining ones to be black. From the perspective of the max-player, she cannot directly tell which state is red or black because (a) the difference between black circles and red circles only appear in the min-player's observations, and (b) the max-player cannot see what the min-player observes. As a result, the only useful information for the max-player to figure out which circles are red  is the action picked by the min-player in the final step. 

Now suppose the min-player will play the optimal strategy when she knows which half the environment is in, and pick action $b_0$ when she does not.  
In this case, for the max-player, identifying all the red circles is as hard as learning a bandit with $\Omega(2^{H})$ arms where only one arm has reward $1/2$ and all other arms has reward $0$.
Therefore, by using standard lower bound arguments for bandits, we can show the max-player's cumulative rewards in the first $K=\Theta(2^{H})$ episodes is $0$ with constant probability. In comparison, the optimal strategy, which avoids visiting all red circles, can collect  $K/3$ rewards with high probability. As a result, we obtain the desired $\Omega(\min\{2^H,K\})$ regret lower bound for competing against the Nash value.

\subsection{Playing against adversary in multi-step weakly-revealing POMGs is hard}\label{sec:over-hard}

In this section, we prove that competing with the max-min value is statistically hard even if (i) the POMG is two-player zero-sum and satisfies Assumption \ref{asp:over} with $m=2$ and $\alpha=1$, (ii) the opponent keeps playing a fixed action, and (iii) the player can directly observe the opponents' actions and observations.

\begin{theorem}
    Assume the player can directly observe the opponents' actions and observations.
    For any $L,k\in\N^+$, there exist (i) a  two-player zero-sum POMG  of size $S,A,O,H=\cO(L)$ and satisfying Assumption \ref{asp:over} with $m=2$ and $\alpha=1$, and (ii) an opponent who keeps playing a fixed action $\hat{a}_2$, so that with probability at least $1/2$
    $$\textstyle
    \sum_{t=1}^k \left(\max_{\tilde \pi_1}\min_{\tilde \pi_2} V_1^{\tilde\pi_1 \times \tilde\pi_2} - V_{1}^{\pi^t_1 \times \hat{a}_2}\right) \ge \Omega\left(\min\{2^L,k\} \right),
    $$
    where $\pi_1^t$ is the policy played by the learner in the $t^{\rm th}$ episode. 
    \end{theorem}

\begin{proof}
The hard instance is constructed as following:
\begin{itemize}
    \item \textbf{States and actions}: There are four states: $p_0,p_1$ and $q_0,q_1$. 
    Each player has two actions,  denoted by $\{a_0,a_1\}$ and $\{b_0,b_1\}$ respectively.
    \item \textbf{Emission and reward}: There are three different observations: $o_{\rm dummy}$, $o_1$ and $o_0$. At step $h\in[H-1]$, $p_1$ and $p_0$ emit the same observation $o_{\rm dummy}$. At step $H$, $p_1$ emits $o_1$ while $p_0$ emits $o_0$. 
    Regardless of $h$, $q_0$ always emits $o_0$ and $q_1$ always emits $o_1$. Importantly, all players share the same observation. 
    The reward function is defined so that $r(o_{\rm dummy})=r(o_{0})=0$ and $r(o_1)=1$ for the max-player. Since the game is zero-sum, the reward function for the min-player is simply $-r(\cdot)$.
    \item \textbf{Transition}: Let $x_1,\ldots,x_h$ be a binary sequence sampled independently and uniformly at random from standard Bernoulli distribution.  At step $h=1$, the POMG always starts from state $p_1$. For each step $h\in[H-1]$: 
    \begin{itemize}
        \item If the current state is $p_i$, then the environment will transition to $p_1$ if and only if $i=1$, the max-player plays action $a_{x_h}$, and the min-player plays $b_0$. Otherwise, if the min-player plays $b_1$,  then the environment will transition to $q_i$. Otherwise, the environment will transition to $p_0$.
        \item If the current state is $q_i$, the next state will be $q_1$ regardless of players' actions.
    \end{itemize}
\end{itemize}
We have the following observations: 
\begin{itemize}
    \item If the min-player keeps playing $b_0$, then from the perspective of the max-player the POMG essentially reduces to a multi-arm bandit problem with $2^{H-1}$ arms because in this case the only useful feedback for the max-player is the reward observed at step $H$.
    \item The max-min (Nash) value is equal to $1$, which is attained when the max-player  picks $a_{x_h}$ at step $h$ with probability $1$.
    \item The $2$-step emission-action matrix at each step $h\in[H-1]$ is rank $4$ and has minimum singular value no smaller than $1$, because we can always exactly identify the current state (for step $h\in[H-1]$) by the current-step observation and the next-step observation if the min-player picks action $b_1$ in the current step.
\end{itemize}
Based on the first two observations above, we immediately obtain a $\Theta(\min\{2^H,k\})$ lower bound for competing with the max-min (Nash) value. Using the third observation, we know the POMG is $2$-step weakly revealing with $\alpha=1$, which completes the proof.
\end{proof}


\end{document}